%% file: MCMCv15submission.tex
\documentclass[sn-mathphys,Numbered]{sn-jnl}

\input{preamble}
\usepackage{pdfpages}
\usepackage{algorithm}
\usepackage{algorithmic}
\usepackage{caption}
\usepackage{subcaption}

\usepackage{graphicx}%
\usepackage{multirow}%
\usepackage{amsmath,amssymb,amsfonts}%
\usepackage{amsthm}%
\usepackage{mathrsfs}%
\usepackage[title]{appendix}%
\usepackage{xcolor}%
\usepackage{textcomp}%
\usepackage{manyfoot}%
\usepackage{booktabs}%
\usepackage{listings}%

\newcommand{\bQ}{\mathbf{Q}}

\newcommand{\calA}{\mathcal{A}}

\newcommand{\calE}{\mathcal{E}}

\newcommand{\calS}{\mathcal{S}}

\newcommand{\calM}{\mathcal{M}}

\newcommand{\bbR}{\mathbb{R}} 
\newcommand{\bbE}{\mathbb{E}} 
\newcommand{\bbP}{\mathbb{P}} 

\theoremstyle{thmstyleone}%
\newtheorem{theorem}{Theorem}
\theoremstyle{thmstyletwo}%
\newtheorem{lemma}{Lemma}%
\newtheorem{remark}{Remark}%

\theoremstyle{thmstylethree}%
\newtheorem{definition}{Definition}%

\begin{document}
	
	\title[Article Title]{Variable-Complexity Weighted-Tempered Gibbs Samplers for Bayesian Variable Selection}
	
	
	\author{\fnm{Lan V.} \sur{Truong}}\email{lt407@cam.ac.uk}
	
	
	
	\affil{\orgdiv{Department of Engineering}, \orgname{University of Cambridge}, \orgaddress{\street{Trumpington}, \city{Cambridge}, \postcode{CB2 1PZ}, \state{Cambridgeshire}, \country{United Kingdom}}}
	
	
	
	
	\abstract{Subset weighted-Tempered Gibbs Sampler (wTGS) has been recently introduced by Jankowiak to reduce the computation complexity per MCMC iteration in high-dimensional applications where the exact calculation of the posterior inclusion probabilities (PIP) is not essential. However, the Rao-Backwellized estimator associated with this sampler has a high variance as the ratio between the signal dimension and the number of conditional PIP estimations is large. In this paper, we design a new subset weighted-Tempered Gibbs Sampler (wTGS) where the expected number of computations of conditional PIPs per MCMC iteration can be much smaller than the signal dimension. Different from the subset wTGS and wTGS, our sampler has a variable complexity per MCMC iteration. We provide an upper bound on the variance of an associated Rao-Blackwellized estimator for this sampler at a finite number of iterations, $T$, and show that the variance is $O\big(\big(\frac{P}{S}\big)^2 \frac{\log T}{T}\big)$ for a given dataset where $S$ is the expected number of conditional PIP computations per MCMC iteration. Experiments show that our Rao-Blackwellized estimator can have a smaller variance than its counterpart associated with the subset wTGS. }

	
	
	
	\maketitle

\section{Introduction} \label{intro}
Markov chain Monte Carlo (MCMC) methods comprise a class of algorithms for sampling from a known function. MCMC methods are primarily used for calculating numerical approximations of multi-dimensional integrals, for example in Bayesian statistics, computational physics \citep{Kasim2019RetrievingFF}, computational biology, \citep{Gupta2014ComparisonOP}, and linear models \citep{ Truong2020ReplicaAO, Truong2021LinearMW, Truong2022OnLM}. Monte Carlo algorithms have been very popular over the last decade \citep{Hesterberg2002MonteCS,Robert2005MonteCS}.  Many
practical problems in statistical signal processing, machine learning and statistics, demand fast and accurate procedures for drawing samples from probability distributions that exhibit arbitrary,
non-standard forms \citep{Andrieu2004AnIT,Fitzgerald2001MarkovCM,Read2012EfficientMC}. One of the most popular Monte Carlo methods are the families of Markov chain Monte Carlo (MCMC) algorithms \citep{Andrieu2004AnIT,Robert2005MonteCS} and particle filters \citep{Bugallo2007PerformanceCO}. The MCMC techniques generate a Markov chain with a
pre-established target probability density function as invariant density \citep{Liang2010AdvancedMC}.

Gibbs sampler (GS) is a Markov chain Monte Carlo (MCMC) algorithm for obtaining a sequence of observations which are approximated from a specific multivariate probability distribution. This sequence can be used to approximate the joint distribution, the marginal distribution of one of the variables, or some subset of the variables. It can be also used to compute the expected value (integral) of one of the variables \citep{Bishop,Bolstad}. GS is applicable when the joint distribution is not known explicitly or is difficult to sample from directly, but the conditional distribution of each variable is known and is easy (or at least, easier) to sample from. The GS algorithm generates an instance from the distribution of each variable in turn, conditional on the current values of the other variables. It can be shown that the sequence of samples constitutes a Markov chain, and the stationary distribution of that Markov chain is just the sought-after joint distribution.

GS is commonly used as a means of statistical inference, especially Bayesian inference. However, pure Markov chain based schemes (i.e., ones which simulate from precisely the right target distribution with no need for subsequent important sampling correction) have been far more successful. This is because MCMC methods are usually much more scalable to high-dimensional situations, whereas  important sampling weight variances tend to grow (often exponentially) with dimension. \citep{Zanella2019ScalableIT} proposed a natural way to combine the best of MCMC and importance sampling in a way that is robust in high-dimensional contexts and ameliorates the slow mixing which plagues many Markov chain based schemes. The proposed scheme is called Tempered Gibbs Sampler (TGS), involving component-wise updating rule like Gibbs Sampling (GS), with improved mixing properties and associated importance weights which remain stable as dimension increases. Through an appropriately designed tempering mechanism, TGS circumvents the main limitations of standard GS, such as the slow mixing introduced by strong posterior correlations. It also avoids the requirement to visit all coordinates sequentially, instead iteratively making  state-informed decisions as to which coordinate should be next updated. 

TGS has been applied to Bayesian Variable Selection (BVS) problem, observing multiple orders of magnitude improvements compared to alternative Monte Carlo schemes \citep{Zanella2019ScalableIT}. Since TGS updates each coordinate with the same frequency, in a BVS context, this may be inefficient as the resulting sampler would spend most iterations updating variables that have low or negligible posterior inclusion probability, especially when the signal dimension, $P$, gets large. A better solution, called weighted Tempered Gibbs Sampling (wTGS) \citep{Zanella2019ScalableIT}, updates more often components with a larger inclusion probability, thus having a more focused computational effort. However, despite the intuitive appeal of this approach to BVS problem, approximating the resulting posterior distribution can be computationally challenging. A principal reason for this is the astronomical size of the model space that results whenever there more than a few dozen covariates. To scale the high-dimensional regime, \citep{Jankowiak2022BayesianVS} has recently introduced an efficient MCMC scheme whose cost per iteration can be significantly reduced compared to wTGS. The main idea is to introduce an auxiliary variable that controls which conditional posterior inclusion probabilites (PIPs) are computed in a given MCMC iteration. However, this scheme contains some weaknesses such as the Rao-Blackwellized estimator associated with this sampler has very high variance as $P/S$ large at a small number of MCMC iterations, $T$. In addition, generating an auxiliary random set which is uniformly distributed over ${P\choose S}$ subsets in the subset wTGS algorithm \citep{Jankowiak2022BayesianVS} requires very long running time.  

In this paper, we design a new subset wTGS called variable-complexity wTGS (VC-wTGS). The Rao-Blackwellized estimator associate with VC-wTGS is shown to have the variance $O\big(\frac{\log T}{T} \big(\frac{P}{S}\big)^2 \big)$, where $P,S, T$ are the signal dimension, the number of PIP computations per an MCMC iteration, and the number of MCMC iterations, respectively. 
Experiments show that our scheme converges to PIPs very fast for simulated datasets and that the variance of the Rao-Blackwellized estimator can be much smaller than subset wTGS \citep{Jankowiak2022BayesianVS} when $P/S$ is very high for MNIST dataset. 
More specifically, our contributions include:
\begin{itemize}
	\item We propose a new subset wTGS, called VC-wTGS, where the expected number of conditional PIP computations per MCMC can be much smaller than the signal dimension. 
	\item We analyse the variance of an associated Rao-Blackwellized estimator at each finite number of MCMC iterations. We show that this variance is $O\big(\frac{\log T}{T} \big(\frac{P}{S}\big)^2 \big)$ for a given dataset. 
	\item We provide some experiments on a simulated dataset (multivariate Gaussian dataset) and the real dataset (MNIST). Experiments show that our estimator can have better variance than the subset wTGS-based estimator \citep{Jankowiak2022BayesianVS} at high $P/S$ for the same number of MCMC iterations $T$.  
\end{itemize}
\section{Preliminaries} \label{sec1}
\subsection{Mathematical Backgrounds}\label{sec:background}
Let a Markov chain $\{X_n\}_{n=1}^{\infty}$ on a state space $\calS$ with transition kernel $Q(x,dy)$ and the initial state $X_1 \sim \nu$, where $\calS$ is a Polish space in $\bbR$. In this paper, we consider the Markov chains which are irreducible and positive-recurrent, so the existence of a stationary distribution $\pi$ is guaranteed. An irreducible and recurrent Markov chain on an infinite state-space is called Harris chain \citep{TR1979}. A Markov chain is called \emph{reversible} if the following detailed balance condition is satisfied:
\begin{align}
\pi(dx)Q(x,dy)=\pi(dy)Q(y,dx),\qquad \forall x, y \in \calS.
\end{align} 
Define
\begin{align}
d(t)=\sup_{x \in \calS} d_{\rm{TV}}(Q^t(x,\cdot),\pi), \qquad  t_{\rm{mix}}(\eps):=\min\{t: d(t)\leq \eps\},
\end{align}
and
\begin{align}
\tau_{\min}:=\inf_{0\leq \eps\leq 1}t_{\rm{mix}}(\eps)\bigg(\frac{2-\eps}{1-\eps}\bigg)^2,\qquad 
t_{\rm{mix}}:=t_{\rm{mix}}(1/4) \label{deftaumin}.
\end{align}

Let $L_2(\pi)$ be the Hilbert space of complex valued measurable functions on $\calS$ that are square integrable w.r.t. $\pi$. We endow $L_2(\pi)$ with inner product $\langle f,g \rangle:= \int f g^* d\pi$, and norm $\|f\|_{2,\pi}:=\langle f, f\rangle_{\pi}^{1/2}$. Let $E_{\pi}$ be the associated averaging operator defined by $(E_{\pi})(x,y)=\pi(y), \forall x,y \in \calS$, and 
\begin{align}
\lambda=\|Q-E_{\pi}\|_{L_2(\pi)\to L_2(\pi)} \label{defL2gap},
\end{align} where
$
\|B\|_{L_2(\pi)\to L_2(\pi)}=\max_{v: \|v\|_{2,\pi}=1}\|Bv\|_{2,\pi}.
$
$Q$ can be viewed as a linear operator (infinitesimal generator) on $L_2(\pi)$, denoted by $\bQ$, defined as $(\bQ f)(x):=\bbE_{Q(x,\cdot)}(f)$, and the reversibility is equivalent to the self-adjointness of $\bQ$. The operator $\bQ$ acts on measures on the left, creating a measure $\mu \bQ$, that is, for every measurable subset $A$ of $\calS$, $\mu \bQ (A):=\int_{x \in \calS} Q(x,A)\mu(dx)$. For a Markov chain with stationary distribution $\pi$, we define the \emph{spectrum} of the chain as
\begin{align}
S_2:=\big\{\xi \in \bbC: (\xi \bI-\bQ) \enspace \mbox{is not invertible on} \enspace L_2(\pi)\big\}.
\end{align}
It is known that $\lambda=1-\gamma^*$ \citep{Daniel2015}, 
where
\begin{align}
\gamma^*&:=\begin{cases} 1-\sup\{|\xi|: \xi \in \calS_2, \xi \neq 1\},\nn\\
\qquad \qquad \mbox{if eigenvalue $1$ has multiplicity $1$,}\\
0,\qquad\qquad \mbox{otherwise}\end{cases}
\end{align} is the \emph{the absolute spectral gap} of the Markov chain. The absolute spectral gap can be bounded by the mixing time $t_{\rm{mix}}$ of the Markov chain by the following expression:
\begin{align}
\bigg(\frac{1}{\gamma^*}-1\bigg)\log 2 \leq t_{\rm{mix}} \leq \frac{\log (4/\pi_*)}{\gamma_*},	
\end{align}
where $\pi_*=\min_{x\in \calS} \pi_x$ is the \emph{minimum stationary probability}, which is positive if $Q^k>0$ (entry-wise positive) for some $k\geq 1$. See \citep{WK19ALT} for more detailed discussions. In \citep{Combes2019EE, WK19ALT}, the authors provided algorithms to estimate $t_{\rm{mix}}$ and $\gamma^*$ from a single trajectory. 

Define
\begin{align}
\calM_2:=\bigg\{\nu \in \calM(\calS): \bigg\|\frac{dv}{d\pi}\bigg\|_2<\infty\bigg\},
\end{align} where $\|\cdot\|_2$ is the standard $L_2$ norm in the Hilbert space of complex valued measurable functions on $\calS$.  
\subsection{Problem Set-up}
Consider the linear regression with $X \in \bbR^{N \times P}$ and $Y \in \bbR^N$ and define the following space of models:
\begin{itemize}
	\item inclusion variables:  $\gamma_i\sim \texttt{Bern}(h)$
	\item noise variance: $\sigma_{\gamma}^2 \in \texttt{InvGamma}\big(\frac{1}{2}\nu_0,\frac{1}{2}\nu_0 \lambda_0\big)$
	\item coefficients:  $\beta_{\gamma} \sim \calN(0,\sigma_{\gamma}^2 \tau^{-1} \bI_{|\gamma|})$
	\item response:  $Y_n \sim \calN(\beta_{\gamma}.X_{n\gamma}, \sigma_{\gamma}^2)$
\end{itemize} where $i=1,2,\cdots,P$ and $n=1,2,\cdots,N$. Here each $\gamma_i \in \{0,1\}$ controls whether the coefficient $\beta_i$ and the $i$-th covariate are included $(\gamma_i=1)$ or excluded $(\gamma_i=0)$ from the model. In the following, we use $\gamma$ to refer to the vector $(\gamma_1,\gamma_2,\cdots,\gamma_P)$. The hyperparameter $h \in (0,1)$ controls the overall level of sparsity; in particular $hP$ is the expected number of covariates included a priori. The $|\gamma|$ coefficients $\beta_{\gamma} \in \bbR^{|\gamma|}$ are governed by the standard Gaussian prior with precision proportional to $\tau>0$. Here, $|\gamma| \in \{0,1,2,\cdots,P\}$ denotes the total number of included covariates. The response $Y_n$ is generated from a Gaussian distribution with variance governed by an Inverse Gamma prior. Note that we do not include a bias term, but doing so may be desirable in practice. An attractive feature of the model is that it explicitly reasons about variable inclusion and allows us to define \emph{posterior inclusion probabilities} or PIPs, where
\begin{align}
\texttt{PIP}(i):=p(\gamma_i=1|\calD) \in [0,1]
\end{align} and $\calD=\{X,Y\}$ is the observed dataset. 

\section{Main Results}
\subsection{Introduction to Subset wTGS} \label{sub:subsetwTGS}
In this subsection, we review the subset wTGS which was proposed by \citep{Jankowiak2022BayesianVS}. Consider the following (unnormalized) target distribution:
\begin{align}
f(\gamma,i,\calS):=p(\gamma|\calD) \frac{\frac{1}{2}\eta(\gamma_{-i})}{p(\gamma_i|\gamma_{-i},\calD)}\calU(\calS|i,\calA) \label{eq10}.
\end{align}
Here, $\calS$ ranges over all the subsets of $\{1,2,\cdots,P\}$ of size $S$ that also contain a fixed `anchor' set $\calA\subset \{1,2,\cdots,P\}$ if size $A<S$, and  $\eta(\cdot)$ is some weighting functions. Moreover, $U(\calS|i,\calA)$ is the uniform distribution over the all size $S$ subsets of $\{1,2,\cdots,P\}$ that contain both $i$ and $\calA$.

In practice, the set $\calA$ can be chosen during burn-in. Subset wTGS proceeds by defining a sampling scheme for the target distribution \eqref{eq10} that utilizes Gibbs updates w.r.t. $i$ and $\calS$ and Metropolized-Gibbs update w.r.t. $\gamma_i$.
\begin{itemize}
	\item {\bf $i$-updates:} Marginalizing $i$ from \eqref{eq10} yields
	\begin{align}
	f(\gamma,\calS)=p(\gamma|\calD) \phi(\gamma,\calS)
	\end{align}
where we define
\begin{align}
\phi(\gamma,\calS):=\sum_{i\in\calS} \frac{\frac{1}{2}\eta(\gamma_{-i})}{p(\gamma_i|\gamma_{-i},\calD)}\calU(\calS|i,\calA)
\end{align}
and have leveraged that $\calU(\calS|i,\calA)=0$ if $i\notin \calS$. Crucially, computing $\phi(\gamma,\calS)$ is $\Theta(S)$ instead of $\Theta(P)$. We can do Gibbs updates w.r.t. $i$ using the distribution
\begin{align}
f(i|\gamma,\calS)\sim \frac{\eta(\gamma_{-i})}{p(\gamma_i|\gamma_{-i},\calD)}\calU(\calS|i,\calA).
\end{align}
	\item {\bf $\gamma$-updates:} Just as for $wTGS$ we utilized Metropolized -Gibbs updates w.r.t. $\gamma_i$ that result in deterministic flips $\gamma_i \to 1-\gamma_i$. Likewise the marginal $f(i)$ is proportional to $\texttt{PIP}(i)+\frac{\eps}{P}$ so that the sampler focuses computational efforts on large PIP covariates. 
	\item {\bf $\calS$-updates:} $\calS$ is updated with Gibbs moves, $\calS \sim \calU(\cdot|i,\calA)$. For the full algorithm, see the Algorithm 1. 
\end{itemize}

\begin{algorithm}[bht]
	\caption{The Subset $S$-wTGS Algorithm}\label{ALGa}
	\begin{algorithmic}
		\STATE {\bfseries Input:} Dataset $\calD=\{X,Y\}$ with $P$ covariates; prior inclusion probability $h$; prior precision $\tau$; subset size $S$; anchor set size $A$; total number of MCMC iterations $T$; number of burn-in iteration $T_{\rm{burn}}$.
		\STATE {\bfseries Output:} Approximate weighted posterior samples $\{\rho^{(t)},\gamma^{(t)}\}_{t=T_{\rm{burn}}+1}^T$
		\STATE {\bfseries Initializations:} $\gamma^{(0)}=\underbrace{(0,0,\cdots,0)}_{P \enspace \text{covariates}}$ and choose $\calA$ be the $A$ covariate with exhibiting the largest correlations with $Y$. Choose $i^{(0)}$ randomly from $\{1,2,\cdots,P\}$ and $\calS^{(0)} \sim \calU(\cdot|i^{(0)},\calA)$. 
		\FOR{\texttt{$t=1,2,\cdots,T$}} 
		\STATE Estimate $f(j|\gamma^{(t-1)})\gets \phi_{t-1}(\gamma)^{-1}\frac{\frac{1}{2}\eta(\gamma_{-j}^{(t-1)})}{p(\gamma_j^{(t-1)}|\gamma_{-j}^{(t-1)},\calD)}$ for all $j \in [P]$.
		\STATE Sample $i^{(t)} \sim f(\cdot|\gamma^{(t-1)})$ 
		\STATE $\gamma^{(t)} \gets \texttt{flip}(\gamma^{(t-1)}|i^{(t)})$ where $\texttt{flip}(\gamma|i)$ flips the $i$-th coordinate of $\gamma: \gamma_i \gets 1-\gamma_i$. 
		\STATE Sample $\calS^{(t)} \sim \calU(\cdot|i^{(t)},\calA)$
		\STATE Estimate $S$ conditional PIPs $p(\gamma_j^{(t)}|\gamma_{-j}^{(t)},\calD)$ for all $j \in \calS^{(t)}$
		\STATE $\phi_t(\gamma)\gets \sum_{j \in \calS^{(t)}} \frac{\frac{1}{2}\eta(\gamma_{-j}^{(t)})}{p(\gamma_j^{(t)}|\gamma_{-j}^{(t)},\calD)}$
		\STATE Compute the unnormalized  weights $\tilde{\rho}^{(t)}\gets \phi^{-1}(\gamma^{(t)})$
		\IF {$t\leq T_{\rm{burn}}$} 
		\STATE Adapt $\calA$ using some adaptive scheme. 
		\ENDIF
		\ENDFOR
		\FOR{\texttt{$t=1,2,\cdots,T$}}
		\STATE $\rho^{(t)}\gets \frac{\tilde{\rho}^{(t)}}{\sum_{s>T_{\rm{burn}} }^T\tilde{\rho}^{(s)}}$ 
		\ENDFOR  
		\STATE {\bfseries Output:} $\{\rho^{(t)},\gamma^{(t)}\}_{t=1}^T$.
	\end{algorithmic}
\end{algorithm}
The details of this algorithm is described in ALG 1. The associated estimator for this sampler is defined as \citep{Jankowiak2022BayesianVS}:
\begin{align}
\texttt{PIP}(i) \approx \sum_{t=1}^T\rho^{(t)}\big(\bone\{i \in \calS^{(t)}\} p(\gamma_i^{(t)}=1|\gamma_{-i}^{(t)},\calD)+\bone\{i \notin \calS^{(t)}\}\gamma_i^{(t)}\big) \label{defest1}.
\end{align}
\subsection{A Variable Complexity wTGS Scheme}
In the subset wTGS in Subsection \ref{sub:subsetwTGS}, the number of conditional PIP computations per MCMC iteration is fixed, i.e., it is equal to $S$. In the following, we propose a variable-computation complexity-based  wTGS schemes (VC-wTGS), say ALG \ref{ALG1}, where the only requirement is that the expected number of the conditional PIP computations per MCMC iteration is $S$. This means that
$
\bbE[S_t]= S,  
$ where $S_t$ is the number of conditional PIP computations at the $t$-th MCMC iteration. 

Compared with ALG \ref{ALGa}, ALG \ref{ALG1} allows us to use different subset sizes at MCMC iterations. By ALG \ref{ALG1}, the expectation of number of conditional PIP computations in each MCMC iteration is $P \times (S/P)+0 \times (1-S/P)=S$. Since we aim to bound the variance at each finite iteration $T$, we don't mention about $T_{\rm{burn}}$ in ALG \ref{ALG1}. In practice, we usually remove some initial samples. We also use the following new version of Rao-Blackwellized estimator:
\begin{align}
\texttt{PIP}(i) \approx  \sum_{t=1}^T \rho^{(t)} p(\gamma_i^{(t)}=1|\gamma_{-i}^{(t)},\calD) \label{rao-blest}.
\end{align}
\begin{algorithm}[h] 
	\caption{A Variable-Complexity Based  wTGS Algorithm}\label{ALG1}
	\begin{algorithmic}
		\STATE {\bfseries Input:} Dataset $\calD=\{X,Y\}$ with $P$ covariates; prior inclusion probability $h$; prior precision $\tau$; total number of MCMC iterations $T$; subset size $S$.
		\STATE {\bfseries Output:} Approximate weighted posterior samples $\{\rho^{(t)},\gamma^{(t)}\}_{t=1}^T$
		\STATE {\bfseries Initializations:} $\gamma^{(0)}=(\gamma_1,\gamma_2,\cdots,\gamma_P)$ where $\gamma_j  \sim \texttt{Bern}(h)$ for all $j \in [P]$. 		
		\FOR{\texttt{$t=1,2,\cdots,T$}} 
		\STATE Set $Q^{(1)}=1$. Sample a Bernoulli random variable $Q^{(t)} \sim \texttt{Bern}(\frac{S}{P})$ if $t\geq 2$.
		\IF{$Q^{(t)}=1$} 
		\STATE Estimate $f(j|\gamma^{(t-1)})\gets \phi_{t-1}(\gamma)^{-1}\frac{\frac{1}{2}\eta(\gamma_{-j}^{(t-1)})}{p(\gamma_j^{(t-1)}|\gamma_{-j}^{(t-1)},\calD)}$ for all $j \in [P]$.
		\STATE Sample $i^{(t)} \sim f(\cdot|\gamma^{(t-1)})$ 
		\STATE $\gamma^{(t)} \gets \texttt{flip}(\gamma^{(t-1)}|i^{(t)})$ where $\texttt{flip}(\gamma|i)$ flips the $i$-th coordinate of $\gamma: \gamma_i \gets 1-\gamma_i$. 
		\STATE Estimate $P$ conditional PIPs $p(\gamma_j^{(t)}|\gamma_{-j}^{(t)},\calD)$ for all $j \in [P]$
		\STATE $\phi(\gamma^{(t)})\gets \sum_{j \in [P]} \frac{\frac{1}{2}\eta(\gamma_{-j}^{(t)})}{p(\gamma_j^{(t)}|\gamma_{-j}^{(t)},\calD)}$
		\STATE Compute the unnormalized  weights $\tilde{\rho}^{(t)}\gets \phi^{-1}(\gamma^{(t)})$
		\ELSE
		\STATE $\gamma^{(t)} \gets \gamma^{(t-1)}$
		\STATE $\tilde{\rho}^{(t)}\gets 1$
		\ENDIF
		\ENDFOR
		\FOR{\texttt{$t=1,2,\cdots,T$}}
		\STATE $\rho^{(t)}\gets \frac{\tilde{\rho}^{(t)}Q^{(t)}}{\sum_{s=1}^T\tilde{\rho}^{(s)}Q^{(s)}}$ 
		\ENDFOR  
		\STATE {\bfseries Output:} $\{\rho^{(t)},\gamma^{(t)}\}_{t=1}^T$.
	\end{algorithmic}
\end{algorithm}
\begin{remark} In ALG \ref{ALG1}, Bernoulli random variables $\{Q^{(t)}\}_{t=1}^T$ are used to replace for random set $\calS$ in ALG \ref{ALGa}. There are two main reasons for this replacement: (1) generating a random set $\calS$ from ${P \choose S}$ subsets of $[P]$ takes very long running time for most pairs $(P,S)$, (2) the associated Rao-Blackwellized estimator usually has smaller variance with ALG \ref{ALG1} than ALG \ref{ALGa} at high $P/S$. See Section \ref{sub:experiment} for our simulation results. 
\end{remark} 
\subsection{Theoretical Bounds for Algorithm \ref{ALG1}} \label{theore}
First, we prove the following result. The proof can be found in Appendix \ref{lem0:proof}.
\begin{lemma} \label{lem0} Let $U$ and $V$ be two positive random variables such that $U/V\leq M$ a.s. for some constant $M$. In addition, assume that on a set $D$ with probability at least $1-\alpha$, we have
	\begin{align}
	|U-\bbE[U]| &\leq \eps \bbE[U],\\
	|V-\bbE[V]| &\leq \eps \bbE[V],
	\end{align} for some $0\leq \eps<1$.  Then, it holds that 
	\begin{align}
	\bbE\bigg[\bigg|\frac{U}{V}-\frac{\bbE[U]}{\bbE[V]}\bigg|^2\bigg]\leq \frac{4\eps^2}{(1-\eps)^2} \bigg(\frac{\bbE[U]}{\bbE[V]}\bigg)^2+ \bigg[\max\bigg(M, \frac{\bbE[U]}{\bbE[V]}\bigg)\bigg]^2 \alpha.
	\end{align}
\end{lemma}
We also recall the following Hoeffding's inequality for Markov chain:
\begin{lemma}\cite[Theorem 1.1]{Rao2018AHI} \label{hoeffinglem} Let $\{Y_i\}_{i=1}^{\infty}$ be a stationary Markov chain with state space $[N]$, transition matrix $A$, stationary probability measure $\pi$, and averaging operator $E_{\pi}$, so that $Y_1$ is distributed according to $\pi$. Let $\lambda=\|A-E_{\pi}\|_{L_2(\pi) \to L_2(\pi)}$ and let $f_1,f_2,\cdots, f_n: [N] \to \bbR$ so that $\bbE[f_i(Y_i)]=0$ for all $i$ and $|f_i(\nu)|\leq a_i$ for all $\nu \in [N]$ and all $i$. Then for $u\geq 0$, 	
	\begin{align}
	\bbP\bigg[\bigg|\sum_{i=1}^n f_i(Y_i)\bigg|\geq u\bigg(\sum_{i=1}^n a_i^2\bigg)^{1/2}\bigg]\leq 2\exp\bigg(-\frac{u^2(1-\lambda)}{64 e}\bigg). 
	\end{align}
\end{lemma}
Now, the following result can be shown.
\begin{lemma} \label{aux:lem1} Let
\begin{align}
\phi(\gamma):=\sum_{j\in [P]}\frac{\frac{1}{2}\eta(\gamma_{-j})}{p(\gamma_j|\gamma_{-j},\calD)}
\end{align} and define
\begin{align}
f(\gamma):=\phi(\gamma) p(\gamma|\calD).
\end{align}
Then, by ALG \ref{ALG1}, the sequence $\{\gamma^{(t)},Q^{(t)}\}_{t=1}^T$ forms a reversible Markov chain with the stationary distribution proportional to $f(\gamma)q(Q)$ where $q$ is the Bernoulli $(S/P)$ distribution. This Markov chain has transition kernel $K((\gamma,Q)\to (\gamma',Q'))= K^*(\gamma \to \gamma') q(Q')$
where
\begin{align}
K^*(\gamma \to \gamma')=\frac{S}{P}\sum_{j=1}^P f(j|\gamma)\delta(\gamma'-\texttt{flip}(\gamma|j))+ \bigg(1-\frac{S}{P}\bigg)\delta(\gamma'-\gamma).
\end{align} 
\end{lemma}
In the classical wTGS  \citep{Zanella2019ScalableIT}, the Markov chain $\{\gamma^{(t)}\}_{t=1}^T$ also form a Markov chain. However, this Markov chain is different from the Markov chain in Lemma \ref{aux:lem1}. However, the two Markov chains still have the same stationary distribution which is proportional to $f(\gamma)$. See a detailed proof of Lemma \ref{aux:lem1} in Appendix \ref{aux:lem1proof}.

\begin{lemma} \label{rbwest} For the  Rao-Blackwellized estimator in \eqref{rao-blest} which is applied to the output sequence $\{\rho^{(t)},\gamma^{(t)}\}_{t=1}^T$ of ALG \ref{ALG1}, it holds that
	\begin{align}
E_{i,T}:=\sum_{t=1}^T\rho^{(t)} p(\gamma_i^{(t)}=1|\gamma_{-i}^{(t)},\calD) &\to \texttt{PIP}(i), \quad \texttt{as} \quad T \to \infty.
	\end{align} 
\end{lemma}
\begin{proof} 
By Lemma \ref{aux:lem1}, $\{\gamma^{(t)},Q^{(t)}\}_{t=1}^T$ forms a reversible Markov chain with stationary distribution $f(\gamma)/Z_f q(Q)$ where $Z_f=\sum_{\gamma} f(\gamma)$. Hence, by SLLN for Markov chain \citep{Breiman1960TheSL}, for any bounded function $h$, we have
\begin{align}
\frac{1}{T}\sum_{t=1}^T \phi^{-1}(\gamma^{(t)}) Q^{(t)} h(\gamma^{(t)}) &\to \bbE_{q f(\cdot)/Z_f}\big[\phi^{-1}(\gamma)h(\gamma)Q\big]\\
&= \sum_Q q(Q) \sum_{\gamma}\frac{f(\gamma)}{Z_f} \phi^{-1}(\gamma)h(\gamma) Q\\
&= \bigg(\sum_Q q(Q) Q\bigg)\bigg(\sum_{\gamma}\frac{f(\gamma)}{Z_f}\phi^{-1}(\gamma) h(\gamma) \bigg)\\
&= \bbE_q[Q] \frac{1}{Z_f} \sum_{\gamma} p(\gamma|\calD) h(\gamma) \label{X10}\\
&=\frac{S}{P} \frac{1}{Z_f}\sum_{\gamma} p(\gamma|\calD) h(\gamma) \label{U1},
\end{align} where \eqref{X10} follows from $f(\gamma)=p(\gamma|\calD)\phi(\gamma)$. 

Similarly,  we have
\begin{align}
\frac{1}{T}\sum_{t=1}^T Q^{(t)} \phi^{-1}(\gamma^{(t)}) &\to  \bbE_{q f(\cdot)/Z_f}\big[\phi^{-1}(\gamma)
Q\big]\\
&=\sum_Q q(Q)Q \sum_{\gamma} \frac{f(\gamma)}{Z_f} \phi^{-1}(\gamma)\\
&=\bbE_q[Q] \sum_{\gamma} \frac{1}{Z_f} p(\gamma|\calD) \label{X11} \\
&=\frac{S}{P} \frac{1}{Z_f} \label{U2},
\end{align} where \eqref{X11} also follows from $f(\gamma)=p(\gamma|D)\phi(\gamma)$. 
\end{proof}
From \eqref{U1} and \eqref{U2}, we obtain
\begin{align}
\frac{\frac{1}{T}\sum_{t=1}^T \phi^{-1}(\gamma^{(t)}) Q^{(t)} h(\gamma^{(t)})}{\frac{1}{T}\sum_{t=1}^T Q^{(t)} \phi^{-1}(\gamma^{(t)})} \to \sum_{\gamma} p(\gamma|\calD) h(\gamma),
\end{align}
or equivalently
\begin{align}
\sum_{t=1}^T \rho^{(t)} h(\gamma^{(t)}) \to \sum_{\gamma} p(\gamma|\calD) h(\gamma) \label{eq105}
\end{align} as $T \to \infty$. 

Now, by setting $h(\gamma)=p(\gamma_i=1|\gamma_{-i},\calD)$, from \eqref{eq105}, we obtain
\begin{align}
 \sum_{t=1}^T \rho^{(t)}  p(\gamma_i^{(t)}=1|\gamma_{-i}^{(t)},\calD) \to \texttt{PIP}(i)
\end{align} for all $i \in [P]$. 

The following result bounds the variance of PIP estimator at finite $T$. 
\begin{lemma} \label{lem2} For any $\eps \in [0,1]$, let $\nu$ and $\pi$ be the initial and stationary distributions of the reversible Markov sequence $\big\{\big(\gamma^{(t)},Q^{(t)}\big)\big\}$. Define
	\begin{align}
	\hat{\phi}(\gamma):=\frac{\phi^{-1}(\gamma)}{\max_{\gamma}\phi^{-1}(\gamma)} \label{defhphi},
	\end{align} 
and
	\begin{align}
	\eps_0=\frac{P}{\texttt{PIP}(i)\bbE_{\pi}[\hat{\phi}(\gamma)]S} \sqrt{\frac{64 e \log T}{(1-\lambda_{\gamma,Q})T}} \label{epsdef}.
	\end{align}
Then, we have
	\begin{align}
	\bbE\bigg[\bigg|\sum_{t=1}^T\rho^{(t)} p(\gamma_i^{(t)}=1| \gamma_{-i}^{(t)},\calD)-\texttt{PIP}(i) \bigg|^2\bigg] \leq \frac{4\eps_0^2}{(1-\eps_0)^2} \texttt{PIP}^2(i)+ \frac{4P}{S} \frac{1}{\min_{\gamma} \pi(\gamma)T} \to 0 \label{smesxi},
	\end{align} as $T\to \infty$ for fixed $P,S$ and the dataset. Here, $\pi(\gamma)$ is the marginal distribution of $\pi(\gamma,Q)$.
\end{lemma}
\begin{proof}
See Appendix \ref{lem2:proof}. 
\end{proof}

Next, we provide a 
lower bound for $1-\lambda_{\gamma,Q}$. First, we recall the following Dirichlet form on spectral gap.
\begin{definition} Let $f,g: \Omega \to \bbR$. The Dirichlet form associated with a reversible Markov chain $Q$ on $\Omega$ is defined by
\begin{align}
\calE(f,g)&=\langle (\bI-\bQ)f, g \rangle_{\pi}\\
&=\sum_{x\in \Omega} \pi(x)[f(x)-\bQ f(x)] g(x)\\
&=\sum_{x\in \Omega} \pi(x)\bigg[\sum_y Q(x,y)(f(x)-f(y))\bigg] g(x)\\
&=\sum_{x,y \in \Omega \times \Omega}\pi(x)Q(x,y)g(x)(f(x)-f(y)).
\end{align}
\end{definition}
\begin{lemma}\citep{Diaconis1993a} \label{diricletmod}(Variational characterisation) For a reversible Markov chain $Q$ with state space $\Omega$ and stationary distribution $\pi$, it holds that
	\begin{align}
	1-\lambda=\inf_{g:\Omega \to \bbR,\atop \bbE_{\pi}[g]=0, \bbE_{\pi}[g^2]=1} \calE(g,g),
	\end{align} where $\calE(g,g):=\langle (\bI-\bQ)g, g\rangle_{\pi}$.
\end{lemma}
\begin{lemma} \label{sgaplem} The spectral gap $1-\lambda_{\gamma,Q}$ of the reversible Markov chain $\{\gamma^{(t)},Q^{(t)}\}$ satisfies
	\begin{align}
1-\lambda_{\gamma,Q} \geq 1-\frac{S}{P}\lambda_P \geq 1-\frac{S}{P}
	\end{align} where $1-\lambda_P$ is the spectral gap of the reversible Markov chain $\{\gamma^{(t)}\}$ of the wTGS algorithm (i.e. $S=P$). 
\end{lemma}
\begin{proof}
From Lemma \ref{diricletmod} and the fact that $\{\gamma^{(t)},Q^{(t)}\}$ forms a reversible Markov chain with transition kernel $K((\gamma,Q)\to (\gamma',Q'))=K^*(\gamma \to \gamma')q(Q')$, we have
\begin{align}
&1-\lambda_{\gamma,Q}\nn\\
&=\inf_{g(\gamma,Q): \bbE_{\pi}[g]=0, \bbE_{\pi}[g^2]= 1}\langle g, g \rangle_{\pi}-\langle \bK g, g \rangle\\
&=1 -\sup_{g(\gamma,Q): \bbE_{\pi}[g]=0, \bbE_{\pi}[g^2]= 1}\langle \bK g, g \rangle\\
&=1- \sup_{g(\gamma,Q): \bbE_{\pi}[g]=0, \bbE_{\pi}[g^2]= 1}\sum_{\gamma,Q} \bK g(\gamma,Q) g(\gamma,Q)\pi(\gamma,Q) \\
&=1-\sup_{g(\gamma,Q): \bbE_{\pi}[g]=0, \bbE_{\pi}[g^2]= 1}\sum_{\gamma,Q} \sum_{\gamma',Q'} K((\gamma,Q)\to (\gamma',Q')) g(\gamma',Q')g(\gamma,Q)\pi(\gamma,Q)\\
&=1-\sup_{g(\gamma,Q): \bbE_{\pi}[g]=0, \bbE_{\pi}[g^2]= 1}\frac{S}{P}\sum_{\gamma,Q}\sum_{\gamma',Q'}K^*(\gamma\to \gamma') q(Q') g(\gamma',Q')g(\gamma,Q)\pi(\gamma,Q)\\
&=1-\frac{S}{P}\sup_{g(\gamma,Q): \bbE_{\pi}[g]=0, \bbE_{\pi}[g^2]= 1}\sum_{\gamma,Q}\sum_{\gamma',Q'}K^*(\gamma\to \gamma') \frac{f(\gamma)}{Z_f}q(Q) g(\gamma',Q')g(\gamma,Q)q(Q')\\
&=1-\frac{S}{P}\sup_{g(\gamma,Q): \bbE_{\pi}[g]=0, \bbE_{\pi}[g^2]= 1}\sum_{\gamma,\gamma'}K^*(\gamma\to \gamma') \frac{f(\gamma)}{Z_f}\sum_{Q,Q'} g(\gamma',Q')g(\gamma,Q)q(Q)q(Q')\\
&=1-\frac{S}{P}\sup_{g(\gamma,Q): \bbE_{\pi}[g]=0, \bbE_{\pi}[g^2]= 1}\sum_{\gamma,\gamma'}K^*(\gamma\to \gamma') \pi(\gamma) \bigg(\sum_Q g(\gamma,Q)q(Q)\bigg)\bigg(\sum_{Q'}\pi(\gamma',Q')q(Q')\bigg)\\
&=1-\frac{S}{P}\sup_{g(\gamma,Q): \bbE_{\pi}[g]=0, \bbE_{\pi}[g^2]= 1}\sum_{\gamma,\gamma'}K^*(\gamma\to \gamma') \pi(\gamma) h(\gamma) h(\gamma')  \label{laxe}
\end{align} where
\begin{align}
\pi(\gamma)&=\frac{f(\gamma)}{Z_f},\\
Z_f&=\sum_{\gamma} f(\gamma),\\
h(\gamma)&:=\sum_Q g(\gamma,Q)q(Q).  
\end{align}
Observe that
\begin{align}
\bbE_{\pi}[h(\gamma)]&= \sum_{\gamma} h(\gamma)\pi(\gamma)\\
&= \sum_{\gamma} \sum_Q g(\gamma,Q)q(Q) \pi(\gamma)\\
&=\sum_{\gamma,Q} g(\gamma,Q) \pi(\gamma,Q)\\
&=\bbE_{\pi}[g(\gamma,Q)]\\
&=0 \label{mu1}. 
\end{align}
On the other hand, we also have
\begin{align}
\bbE_{\pi}\big[h^2(\gamma)\big]&=\sum_{\gamma}\bigg(\sum_Q g(\gamma,Q)q(Q)\bigg)^2 \pi(\gamma)\\
&\leq \sum_{\gamma}\bigg(\sum_Q g(\gamma,Q)^2 q(Q)\bigg) \pi(\gamma) \label{y1}\\
&= \sum_{\gamma,Q} g(\gamma,Q)^2 \pi(\gamma,Q)\\
&= \bbE_{\pi} \big[g(\gamma,Q)^2\big]\\
&=1 \label{mu3}, 
\end{align} where \eqref{y1} follows from the convexity of the function $x^2$ on $[0,\infty)$.

From \eqref{mu1}, \eqref{mu3}, and \eqref{laxe}, we obtain
\begin{align}
1-\lambda_{\gamma,Q}&\geq 1-\frac{S}{P}\sup_{h(\gamma): \bbE_{\pi}[h]=0, \bbE_{\pi}[h^2]\leq 1} \sum_{\gamma,\gamma'}K^*(\gamma\to \gamma') \pi(\gamma) h(\gamma) h(\gamma') \label{me1}.
\end{align}
Now, note that $\bbE_{\pi}[h]=0$ is equivalent to $h \perp_{\pi} \bone$. Let $|\Omega|=2^{P+1}:=n$ and $h_1,h_2,\cdots,h_n$ are eigenfunctions of $\bK^*$ corresponding to the decreasing ordered eigenvalues $\lambda_1 \geq \lambda_2\geq  \cdots \geq \lambda_n$ and are orthogonal since $\bK^*$ is self-adjoint. Set $h_1=\bone$. Since $\|h\|_{2,\pi}=1$ and $h \perp_{\pi} \bone$, we have $h=\sum_{j=2}^n a_j h_j$ because it is perpendicular to $h_1$ so it can be only represented by these eigenvectors. By taking $l_2$-norm on both sizes we have $\sum_{j=2}^n a_j^2\leq 1$ since  the form like $\langle h_i, h_j \rangle_{\pi}=0$ and $\langle h_i, h_i \rangle=\|h_i\|_{2,\pi}^2 = 1$. Thus, 
\begin{align}
\sup_{h: \bbE_{\pi}[h]=0, \bbE_{\pi}[h^2]\leq 1} \sum_{\gamma,\gamma'}K^*(\gamma\to \gamma') \pi(\gamma) h(\gamma) h(\gamma')&\leq \max_{a_2,a_3,\cdots, a_n:\sum_{j=2}^n a_j^2 \leq 1} \sum_{j=1}^n a_j^2 \lambda_j \\
&\leq \lambda_2 \sum_{j=2}^n a_j^2\\
&=\lambda_2 \label{mubai},
\end{align} where $\sum_{j=2}^n a_j^2\leq 1$ and $\lambda_j \in \texttt{spec}(P)$ such that $\lambda_2 \geq \lambda_3 \cdots \geq \lambda_n$. 
Hence, from \eqref{mubai}, we obtain
\begin{align}
1-\lambda_{\gamma,Q}&\geq 1- \frac{S}{P}\lambda_2\\
&=1-\frac{S}{P}\lambda_P \label{umat1}\\
&=\frac{S}{P}\big(1-\lambda_P\big)+1-\frac{S}{P}\\
&\geq 1-\frac{S}{P}.  
\end{align}
\end{proof}
By combining Lemma \ref{rbwest}, Lemma \ref{lem2} and Lemma \ref{sgaplem}, we come up with the following theorem.
\begin{theorem} \label{main:thm} For the variable-complexity subset wTGS-based estimator in \eqref{rao-blest} and given dataset $(X,Y)$, it holds that
\begin{align}
E_{i,T}:=\sum_{t=1}^T\rho^{(t)} p(\gamma_i^{(t)}=1|\gamma_{-i}^{(t)},\calD) &\to \texttt{PIP}(i), \quad \texttt{as} \quad T \to \infty \label{abresul}
\end{align} 
and
\begin{align}
\bbE\bigg[\bigg|\sum_{t=1}^T\rho^{(t)} p(\gamma_i^{(t)}|\gamma_{-i}^{(t)},\calD)-\texttt{PIP}(i) \bigg|^2\bigg] = O\bigg(\frac{\log T}{T} \bigg(\frac{P}{S}\bigg)^2\bigg(\frac{\max_{\gamma}\phi(\gamma)}{\min_{\gamma}\phi(\gamma)}\bigg)^2 \bigg) \label{acresul},
\end{align}
where
\begin{align}
\phi(\gamma)=\frac{1}{2}\sum_{j\in [P]} \frac{p(\gamma_j=1|\gamma_{-j},\calD)}{p(\gamma_j|\gamma_{-j},\calD)}.
\end{align}
\end{theorem} 
\begin{proof}
First, \eqref{abresul} is shown in Lemma \ref{rbwest}. Now, we show \eqref{acresul} by using Lemma \ref{lem2} and Lemma \ref{sgaplem}.

Observe that
\begin{align}
\bbE_{\pi}[\hat{\phi}(\gamma)]&=\bbE_{\pi}\bigg[\frac{\phi^{-1}(\gamma)}{\max_{\gamma}\phi^{-1}(\gamma)}\bigg]\nn\\
&\geq \frac{\min_{\gamma}\phi^{-1}(\gamma)}{\max_{\gamma}\phi^{-1}(\gamma)} \\
&= \frac{\min_{\gamma}\phi(\gamma)}{\max_{\gamma}\phi(\gamma)}
\label{T1}. 
\end{align}
In addition, we have
\begin{align}
\phi(\gamma)&= \sum_{j \in [P]} \frac{\frac{1}{2}\eta(\gamma_{-j})}{p(\gamma_j|\gamma_{-j},\calD)}\\
&= \frac{1}{2}\sum_{j\in [P]} \frac{p(\gamma_j=1|\gamma_{-j},\calD)}{p(\gamma_j|\gamma_{-j},\calD)}.
\end{align}
Now, note that
\begin{align}
\frac{p(\gamma_j=1|\gamma_{-j},\calD)}{p(\gamma_j|\gamma_{-j},\calD)}=\begin{cases} 1,&\qquad  \gamma_j=1 \\\frac{p(\gamma_j=1|\gamma_{-j},\calD)}{p(\gamma_j=0|\gamma_{-j},\calD)},&\qquad  \gamma_j=0.
\end{cases}
\end{align}
In Appendix \ref{posterapp}, we can estimate $\frac{p(\gamma_j=1|\gamma_{-j},\calD)}{p(\gamma_j=0|\gamma_{-j},\calD)}$ based on the dataset. More specifically, let $\tgamma_1$ is given by $\gamma_{-i}$ with $\gamma_i=1$, $\tgamma_0$ is given by $\gamma_{-i}$ with $\gamma_i=0$, then we can show that
\begin{align}
\frac{p(\gamma_j=1|\gamma_{-j},\calD)}{p(\gamma_j=0|\gamma_{-j},\calD)}=\bigg(\frac{h}{1-h}\bigg) \sqrt{\tau \frac{\det(X_{\tgamma_0}^T X_{\tgamma_0} +\tau I)}{\det(X_{\tgamma_1}^T X_{\tgamma_1} +\tau I)}}\bigg(\frac{\|Y\|^2-\|\tilY_{\tgamma_0}\|^2+\nu_0 \lambda_0}{\|Y\|^2-\|\tilY_{\tgamma_1}\|^2+\nu_0 \lambda_0}\bigg)^{\frac{N+\nu_0}{2}} \label{finxo}.
\end{align}
Here, 
\begin{align}
\|\tilY_{\gamma}\|^2&= \tilY_{\gamma}^T \tilY_{\gamma}\\
&= Y^T X_{\gamma} (X_{\gamma}^T X_{\gamma}+\tau I)^{-1} X_{\gamma}^T Y. 
\end{align}
\end{proof}
\section{Experiments} \label{sub:experiment} In this section, we show by simulation that the PIP-estimator is convergent as $T \to \infty$. In addition, we compare the variance of associated Rao-Blackwellized estimators for VC-wTGS and subset wTGS on simulated and real datasets. To compute $p(\gamma_i|\gamma_{-i},Y)$, we use the same trick as \cite[Appendix B.1]{Zanella2019ScalableIT} for the new setting. See our derivations of this posterior distribution in Appendix \ref{posterapp}.
\begin{figure}[h!]
     \centering
     \begin{subfigure}[b]{0.49\textwidth}
         \centering
         \includegraphics[width=\textwidth]{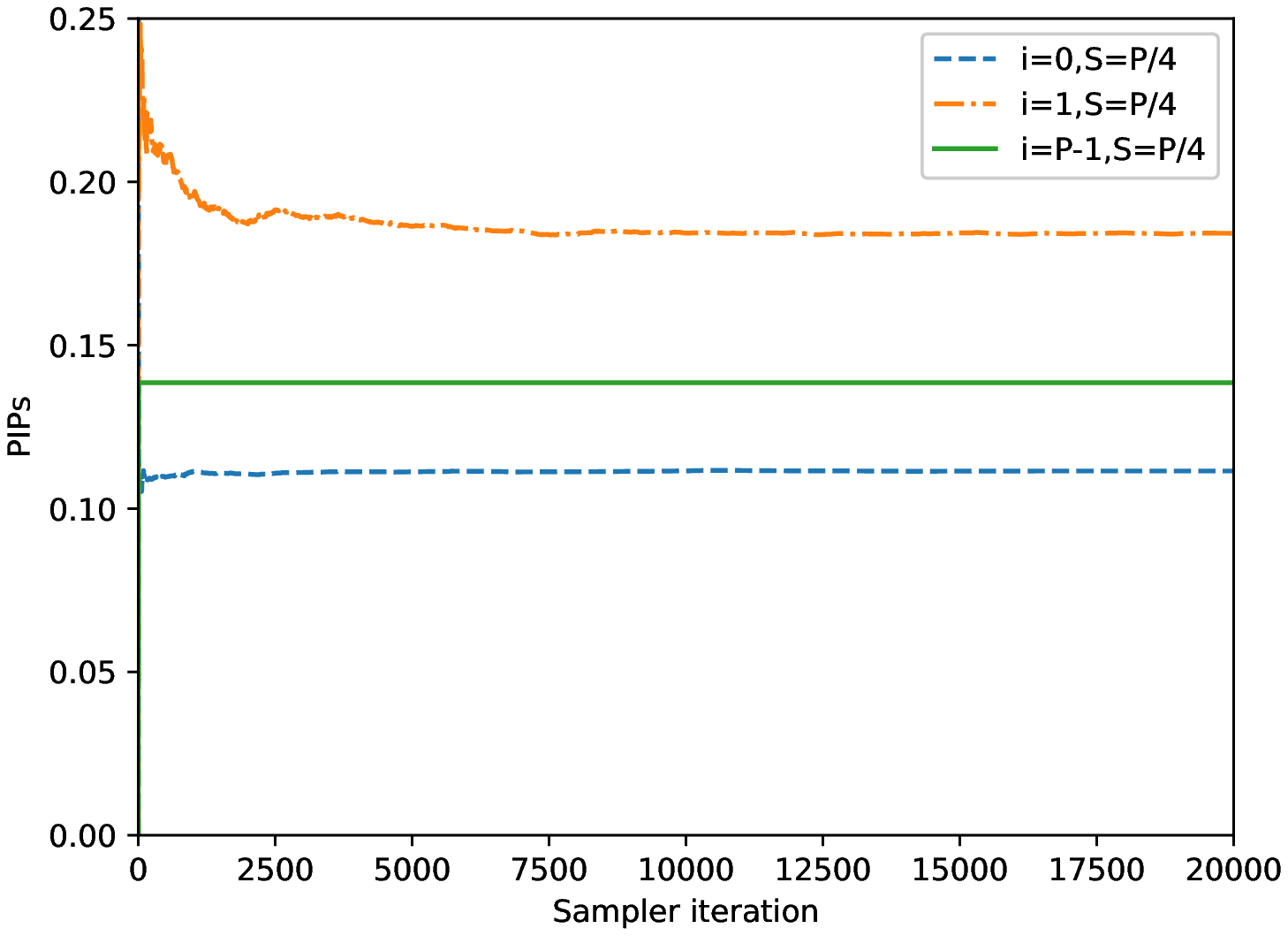}
         \label{fig:1a}
     \end{subfigure}
     \begin{subfigure}[b]{0.49\textwidth}
         \centering
         \includegraphics[width=\textwidth]{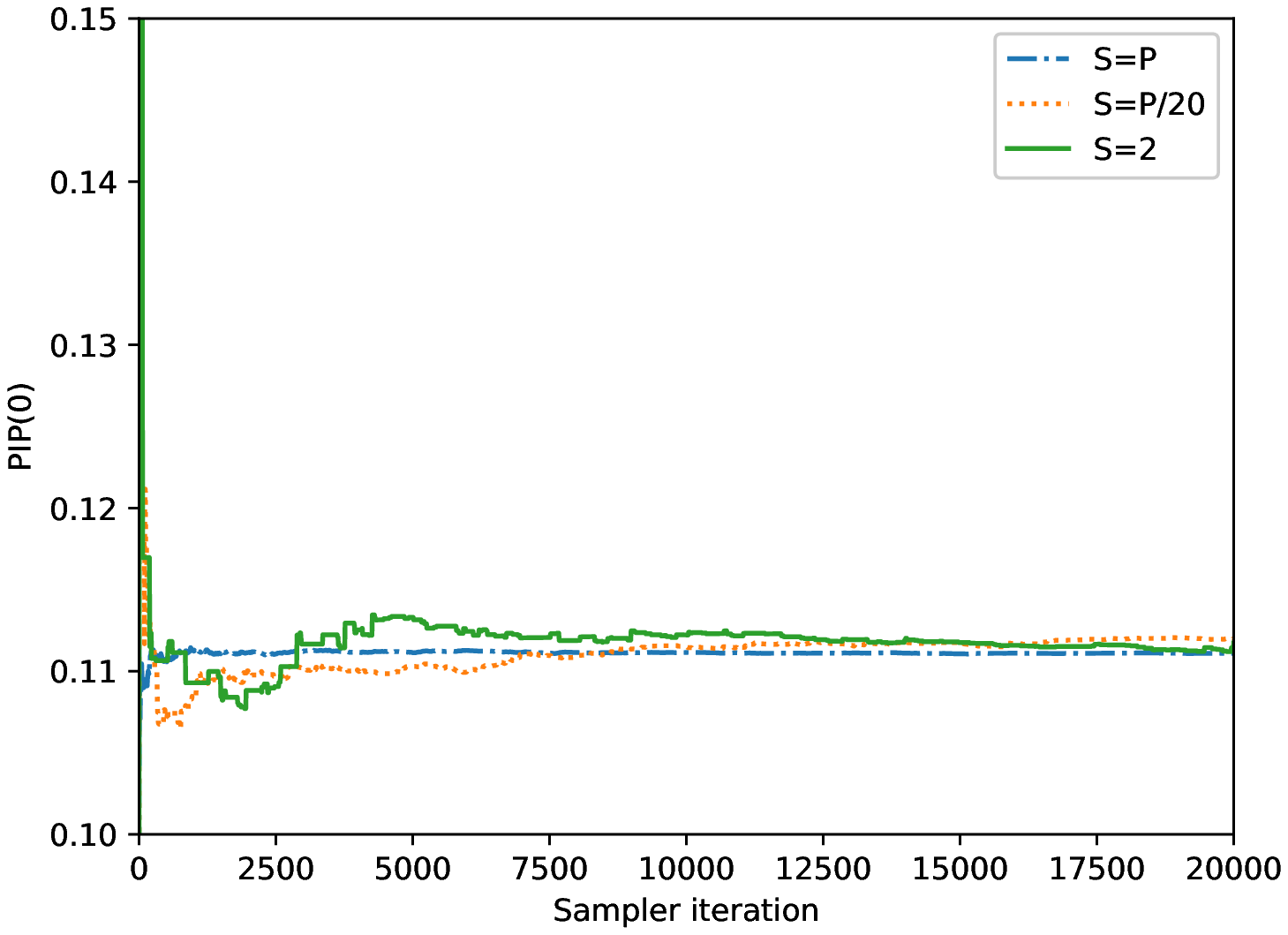}
         \label{fig:2a}
     \end{subfigure}
     \vfill
     \begin{subfigure}[b]{0.49\textwidth}
         \centering
         \includegraphics[width=\textwidth]{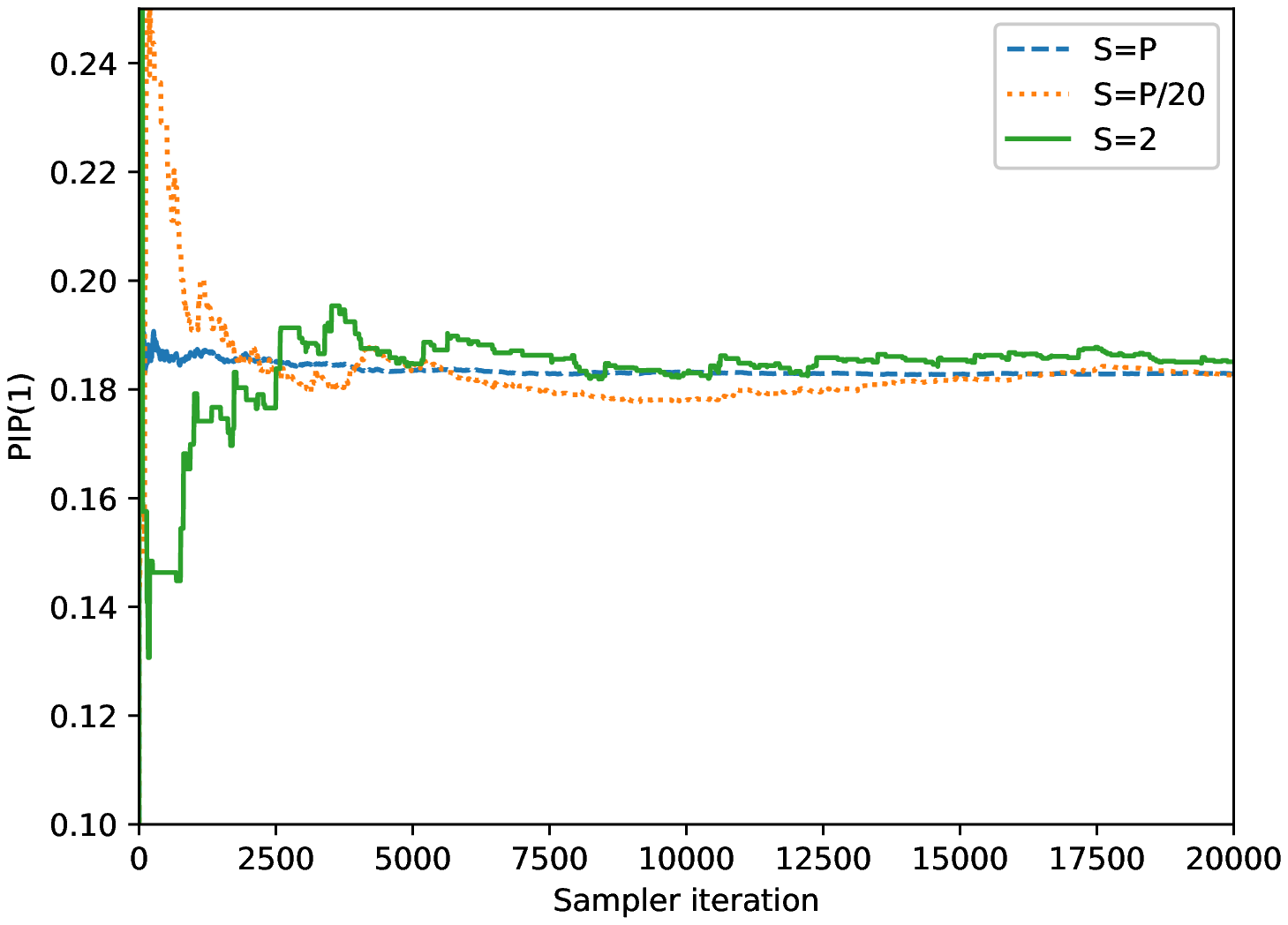}
         \label{fig:3a}
     \end{subfigure}
     \begin{subfigure}[b]{0.49\textwidth}
         \centering
         \includegraphics[width=\textwidth]{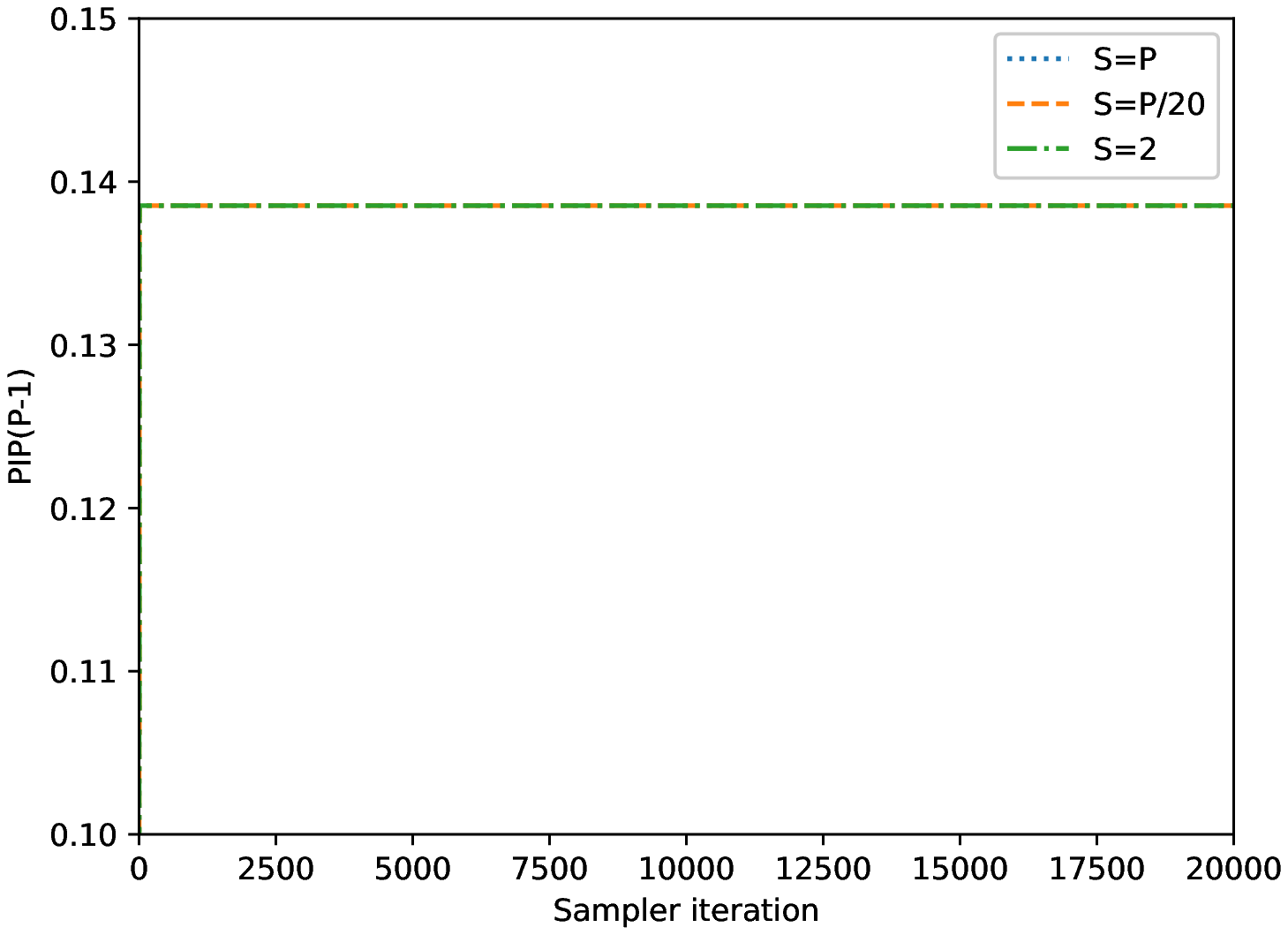}
         \label{fig:4a}
     \end{subfigure}

        \caption{VC-wTGS Rao-Blackwellized Estimators (ALG \ref{ALG1})}
        \label{fig:PIPs}
\end{figure}


\begin{figure}[h!]
     \centering
     \begin{subfigure}[b]{0.49\textwidth}
         \centering
         \includegraphics[width=\textwidth]{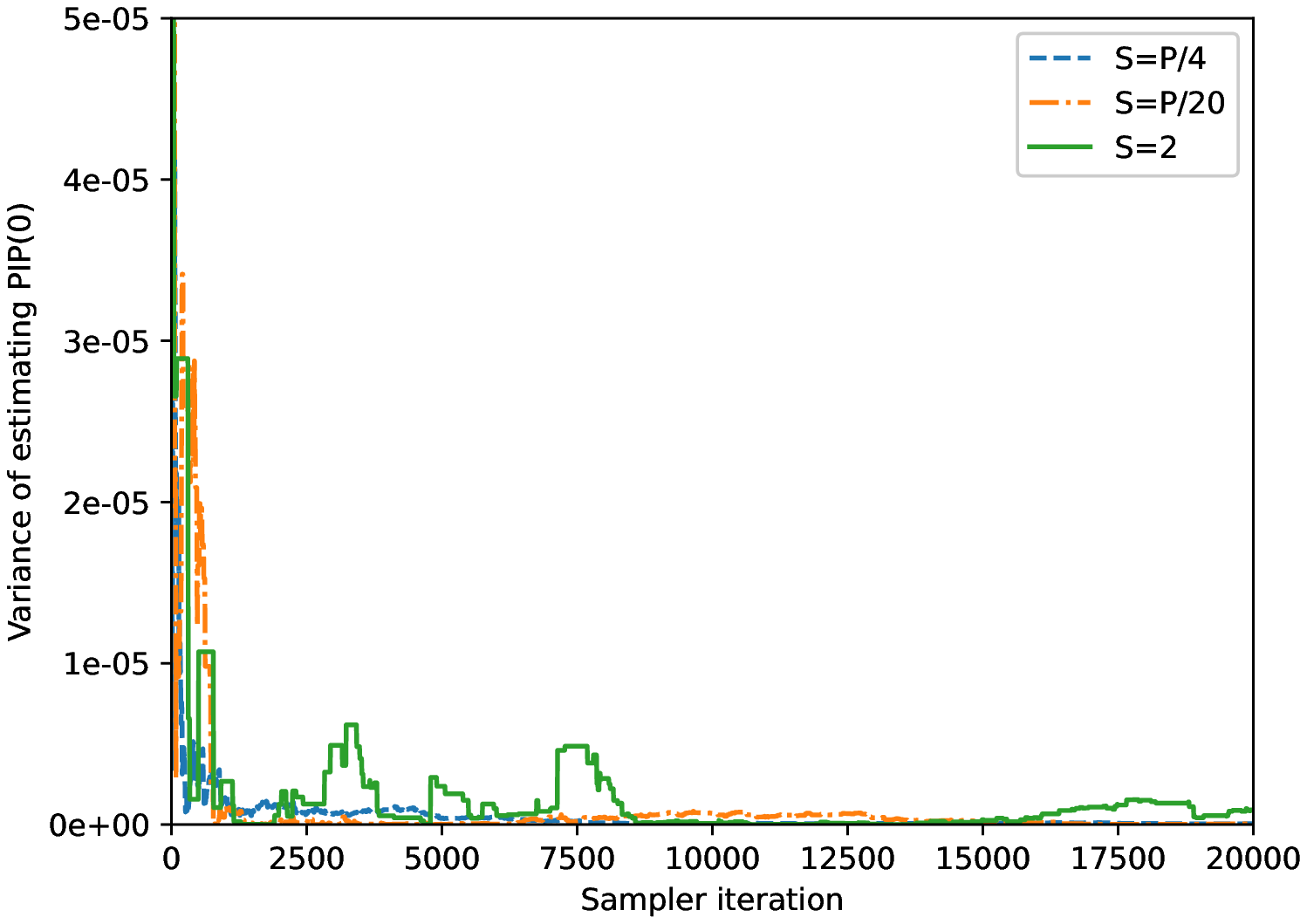}
         \label{fig:5a}
     \end{subfigure}
     \begin{subfigure}[b]{0.49\textwidth}
         \centering
         \includegraphics[width=\textwidth]{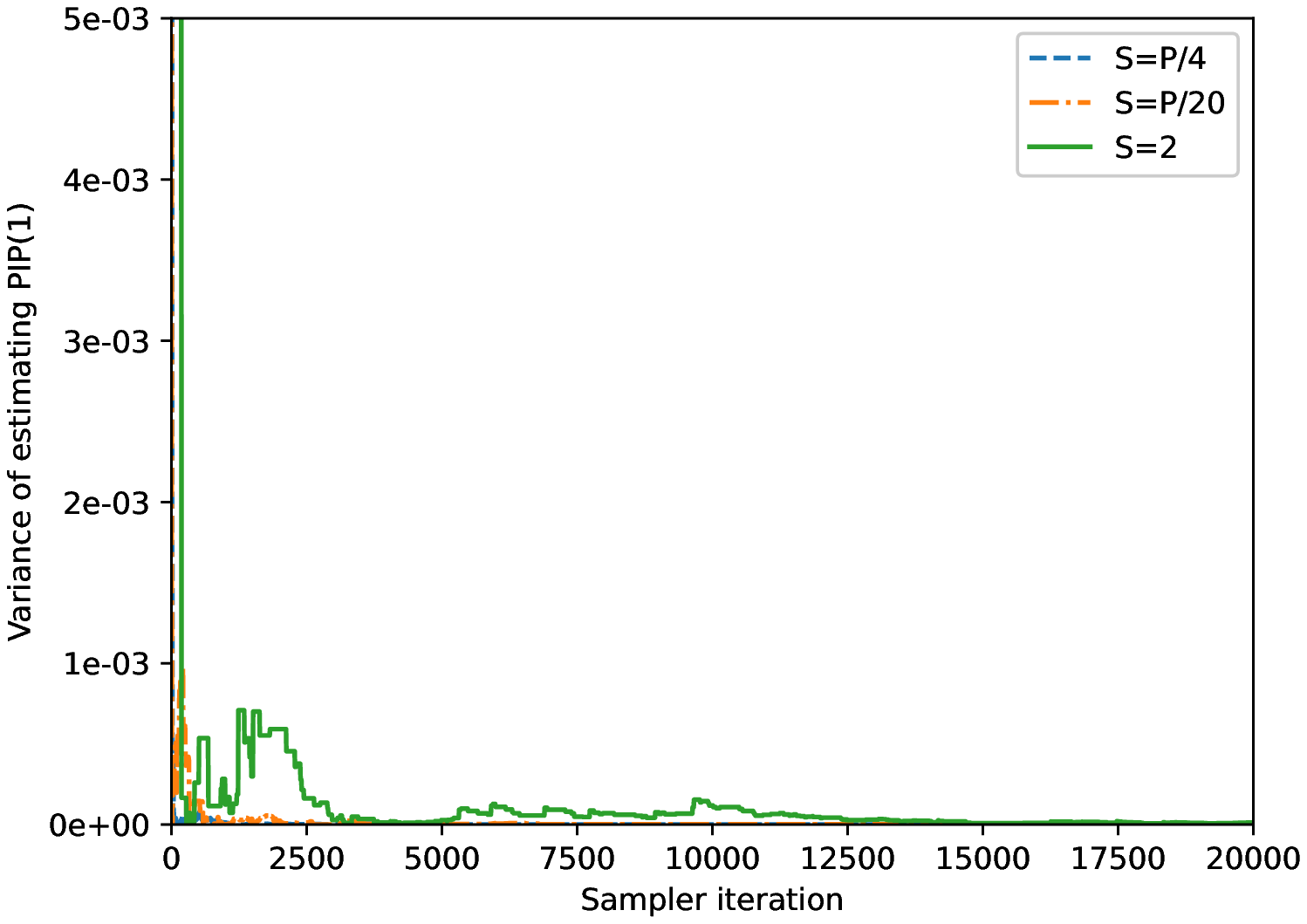}
         \label{fig:5b}
     \end{subfigure}   
        \caption{The variance of VC-wTGS Rao-Blackwellized Estimators}
        \label{fig:variances}
\end{figure}
\subsection{Simulated Datasets}
First, we perform a simulated experiment. Let $X \in \bbR^{N \times P}$ be a realization of a multivariate (random) Gaussian matrix. We consider the case $N=100$ and $P=200$.  We run $T=20000$ iterations. 

Fig.~\ref{fig:PIPs} shows that the Rao-Blackwellized estimator in \eqref{rao-blest} converges to the value of PIP at $T \to \infty$ for different values of $S$. Fig.~\ref{fig:variances} shows the variance for estimating PIP(0) and PIP(1)  at different values of $S$. Since the variance of wTGS is very small $\approx 0$ at $T$ big enough, the variance of variable-complexity wTGS is also small at $T$ big enough. 
\subsection{Real Datasets}
In this simulation, we run ALG 2 on MNIST dataset. Fig.~\ref{fig:PIPsII} and Fig.~\ref{fig:variancesMNIST} plot PIP(0) and PIP(1) and  the estimated variances for the Rao-Blackwellized estimator in \eqref{rao-blest} at different values of $S$, respectively. These plots show a trade-off between the computational complexity and the estimated variance for estimating PIP(0) and PIP(1). The expected number of PIP computations is only $ST$ in ALG \ref{ALG1} but $TP$ in wTGS if we run $T$ MCMC iterations. However, we suffer an increasing in variance. By Theorem \ref{main:thm}, the variance is $O\big(\big(\frac{P}{S}\big)^2 \frac{\log T}{T}\big)$ for a given dataset, i.e., increasing at most $(P/S)^2$ times. For many applications, we don't need to estimate PIPs exactly, hence VC-wTGS can be used to reduce computational complexity especially when $P$ is very large (million covariates). Fig. \ref{fig:PIPvariancesMNIST} shows that VC-wTGS outperforms subset wTGS \citep{Jankowiak2022BayesianVS} at high values of $P/S$. 

\begin{figure}[h!]
	\centering
	\begin{subfigure}[b]{0.49\textwidth}
		\centering
		\includegraphics[width=\textwidth]{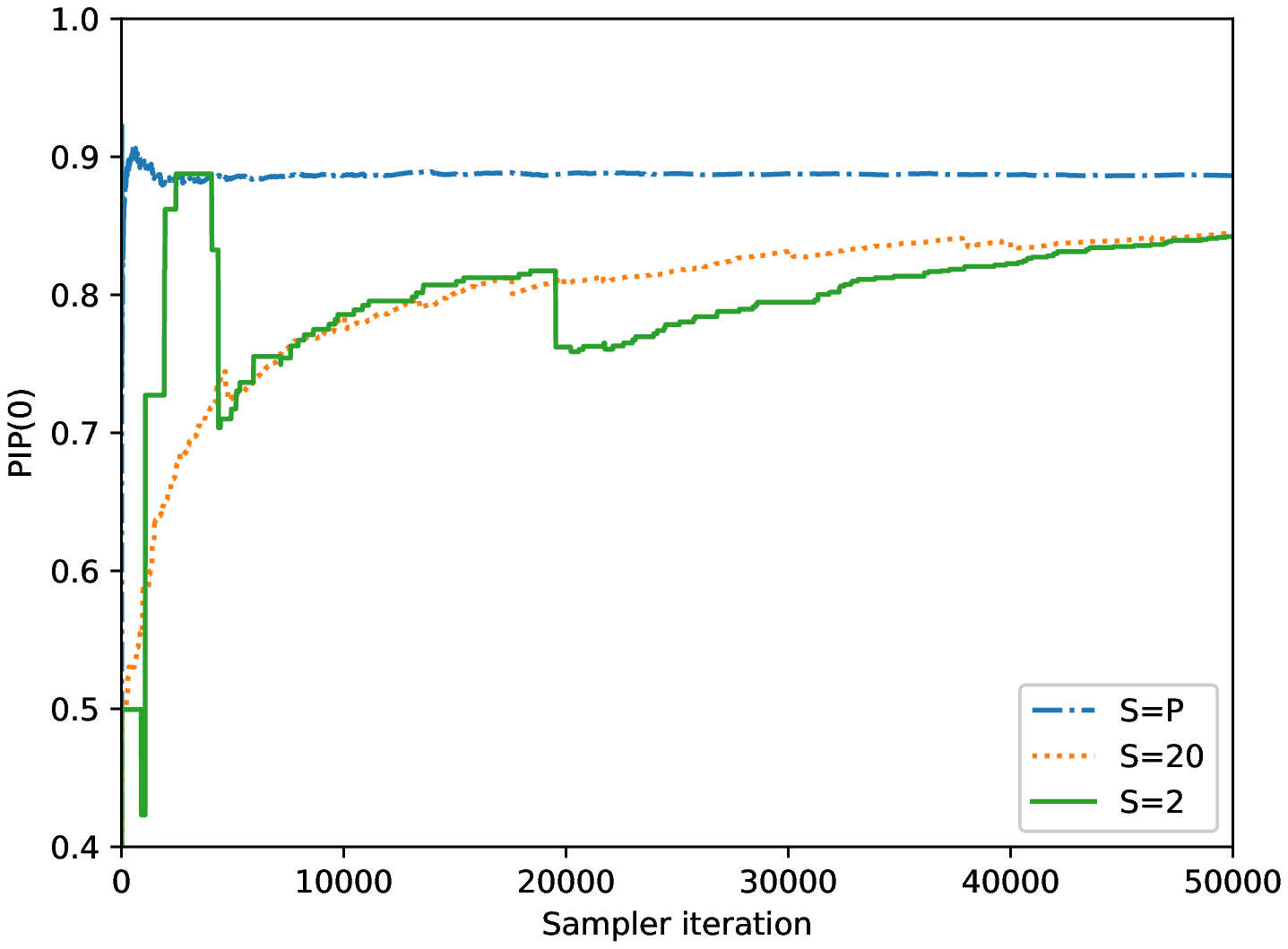}
		\label{fig:2b}
	\end{subfigure}
	\begin{subfigure}[b]{0.49\textwidth}
		\centering
		\includegraphics[width=\textwidth]{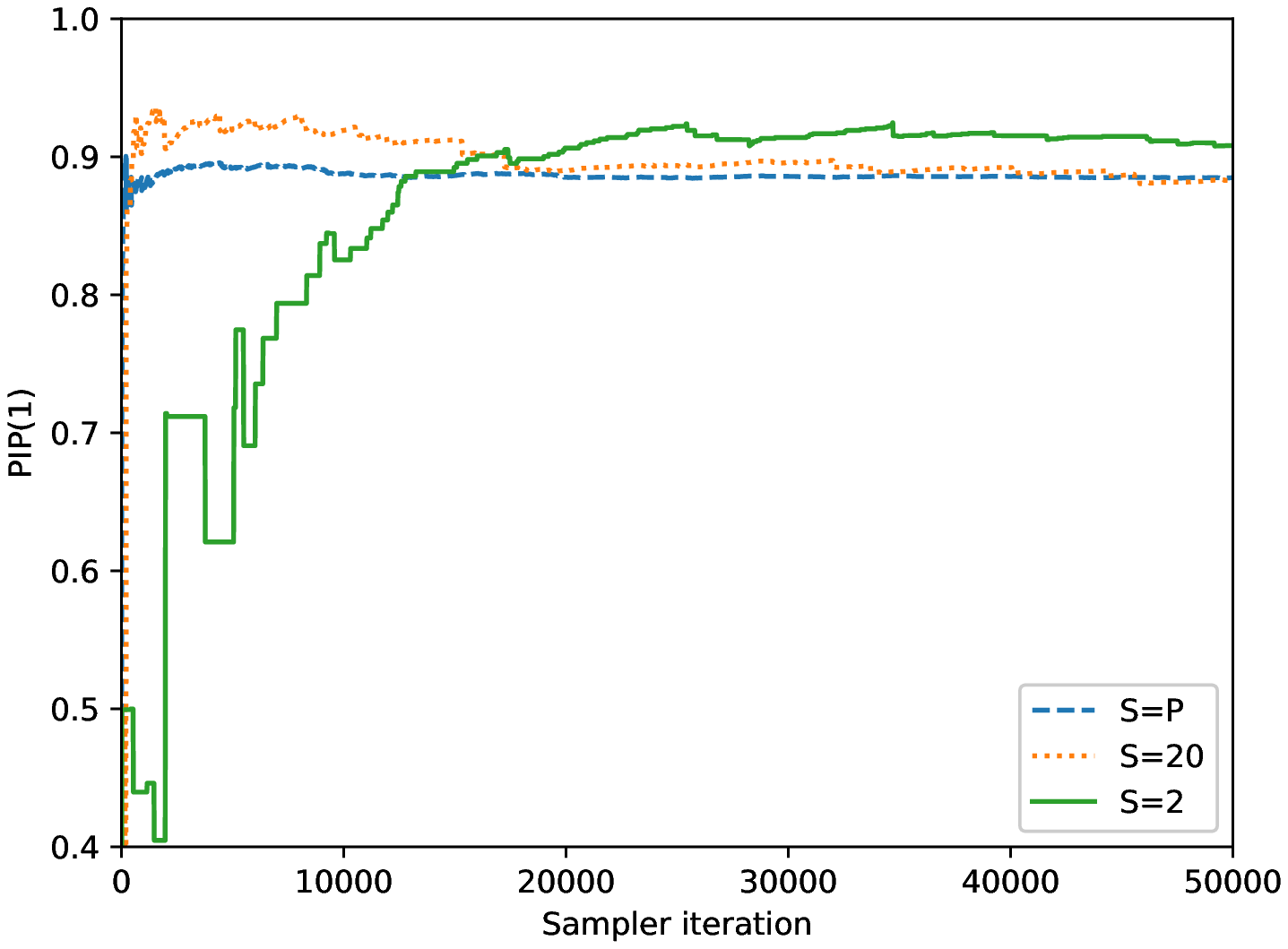}
		\label{fig:3b}
	\end{subfigure}
	\caption{VC-wTGS Rao-Blackwellized Estimators (ALG \ref{ALG1})}
	\label{fig:PIPsII}
\end{figure}

\begin{figure}[h!]
	\centering
	\begin{subfigure}[b]{0.49\textwidth}
		\centering
		\includegraphics[width=\textwidth]{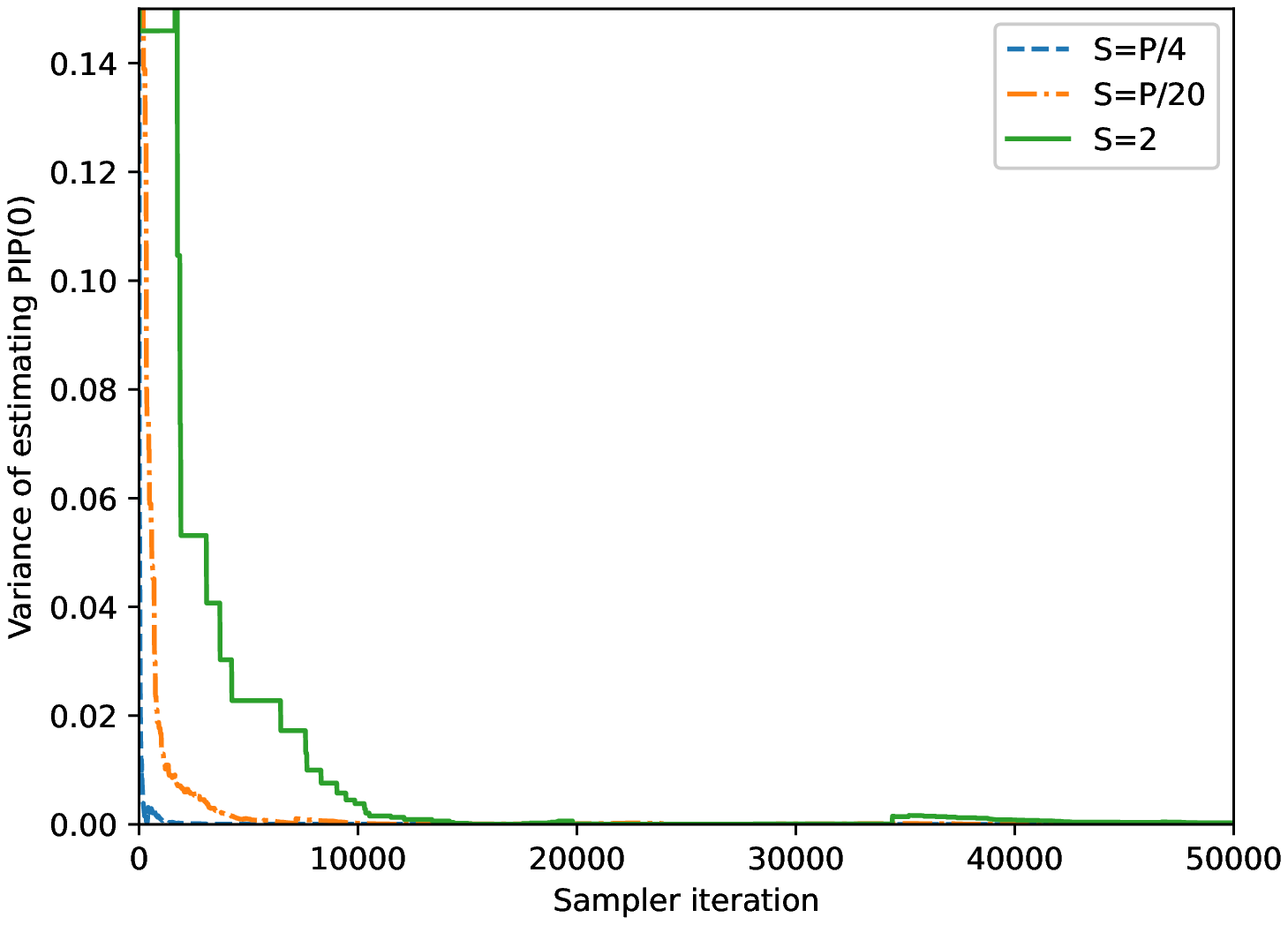}
		\label{fig:5amnist}
	\end{subfigure}
	\begin{subfigure}[b]{0.49\textwidth}
		\centering
		\includegraphics[width=\textwidth]{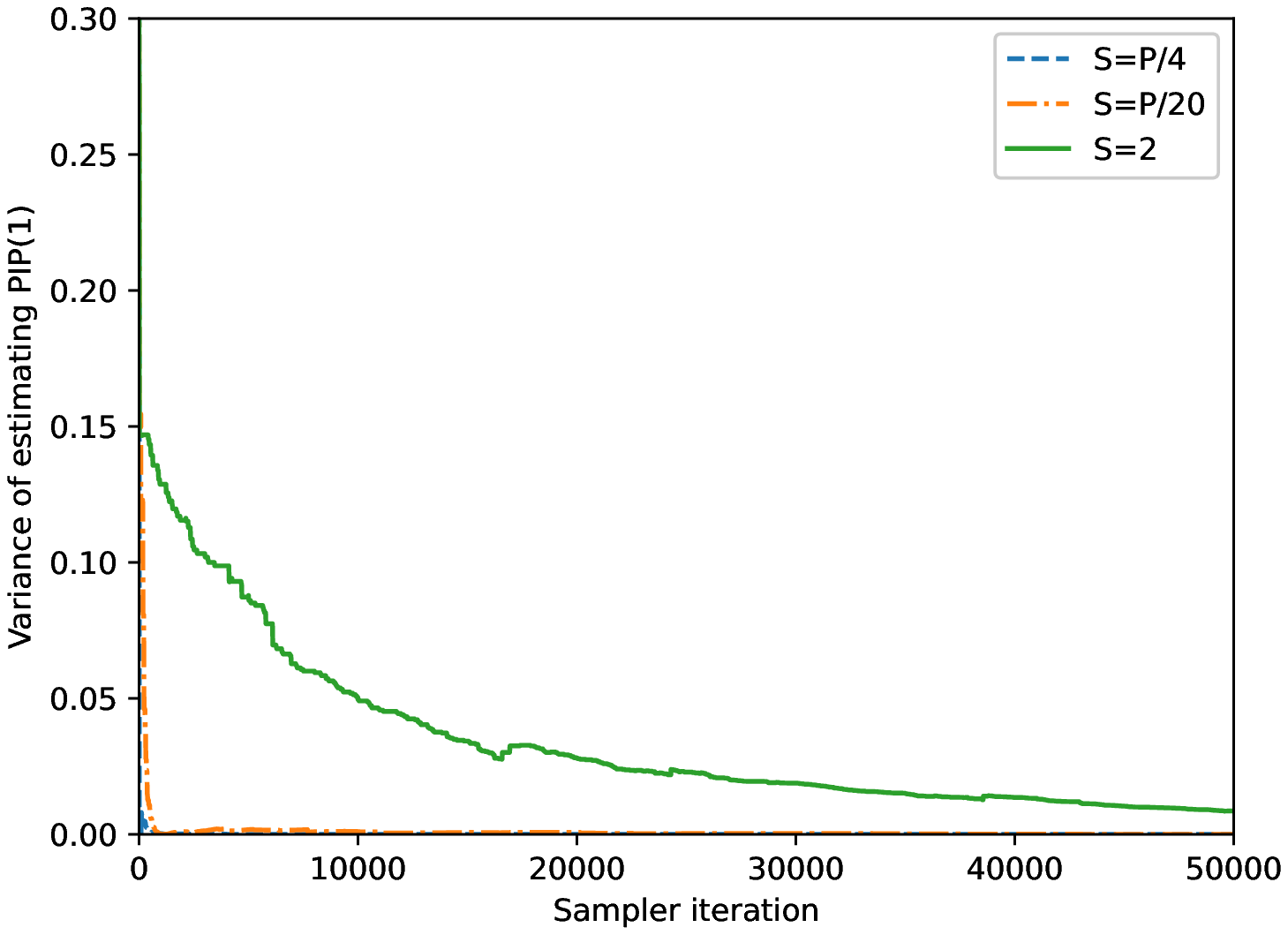}
		\label{fig:5bmnist}
	\end{subfigure}   
	\caption{The variance of VC-wTGS Rao-Blackwellized Estimators}
	\label{fig:variancesMNIST}
\end{figure}
\begin{figure}[h!]
	\centering
	\begin{subfigure}[b]{0.49\textwidth}
		\centering
		\includegraphics[width=\textwidth]{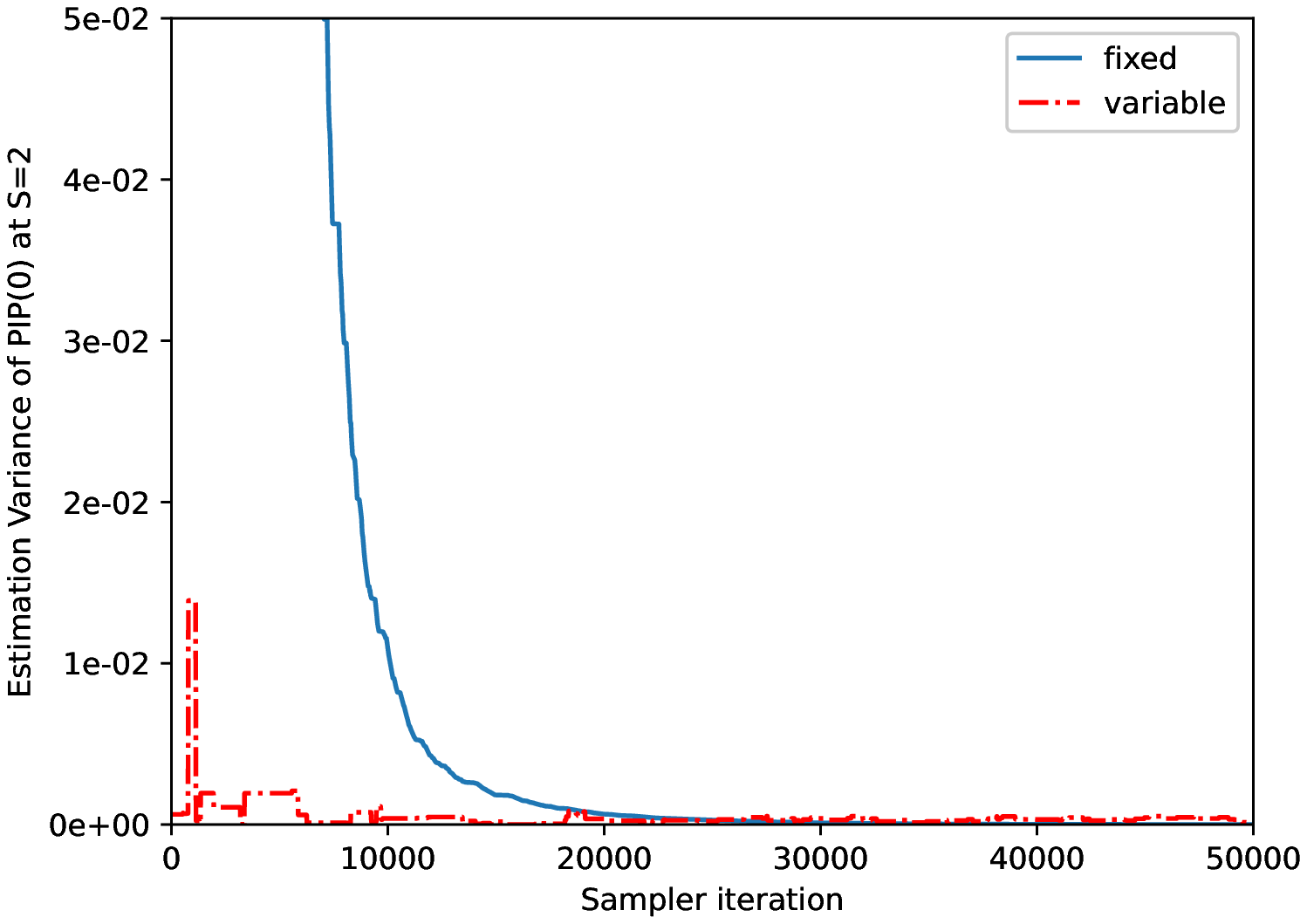}
		\label{fig:6amnist}
	\end{subfigure}
	\begin{subfigure}[b]{0.49\textwidth}
		\centering
		\includegraphics[width=\textwidth]{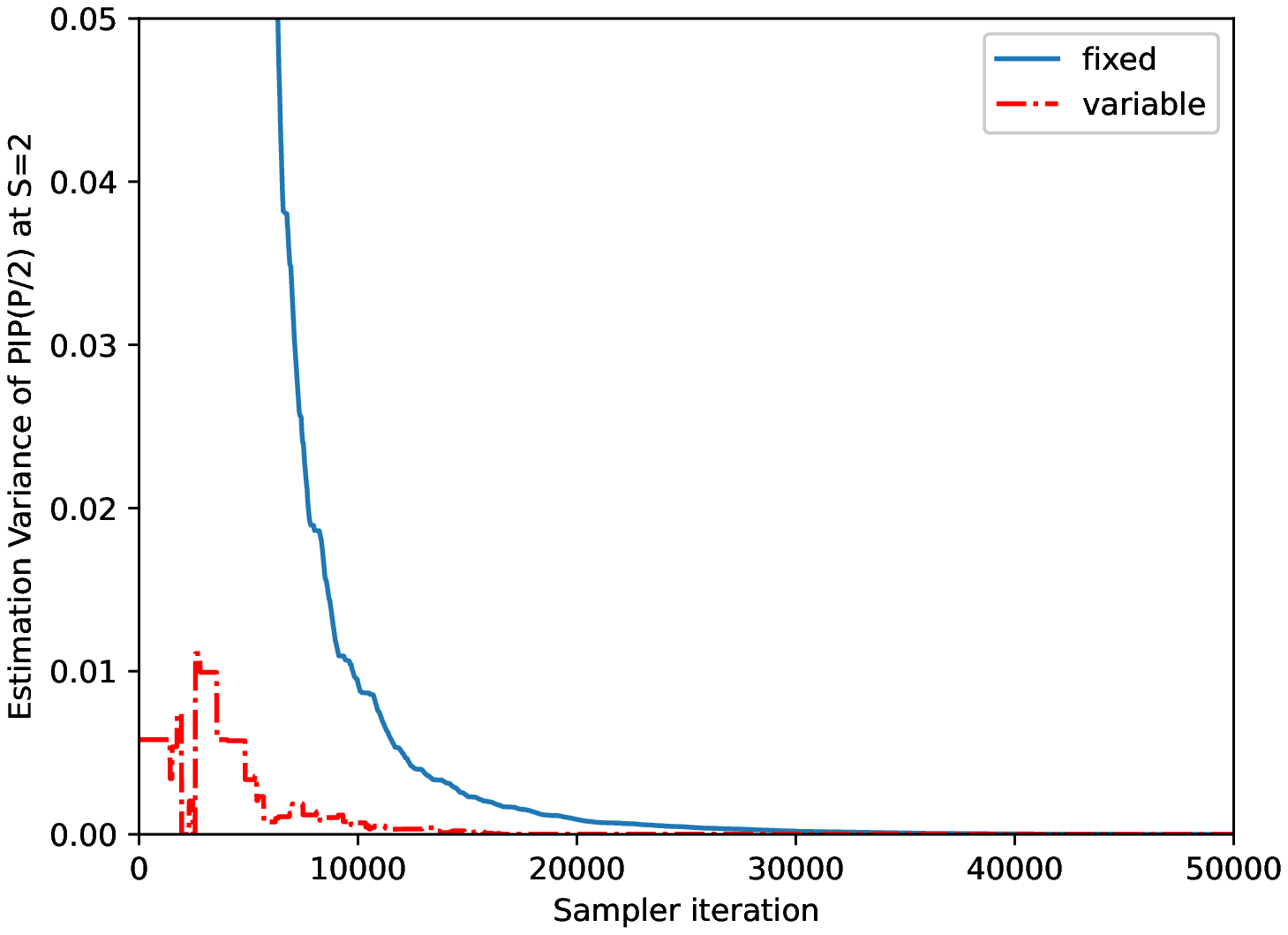}
		\label{fig:6bmnist}
	\end{subfigure}  
	\caption{Comparing the variance between subset wTGS and VC-wTGS at $S=2$.}
	\label{fig:PIPvariancesMNIST}
\end{figure}
\begin{appendices}
\section{Proof of Lemma \ref{lem0}}\label{lem0:proof}
Observe that with probability at least $1-\alpha$, we have
	\begin{align}
	(1-\eps) \bbE[U] &\leq U \leq (1+\eps) \bbE[U]\\
	(1-\eps) \bbE[V] &\leq V \leq (1+\eps) \bbE[V].
	\end{align}
	Hence, we have
	\begin{align}
	\bigg(\frac{1-\eps}{1+\eps}\bigg)\frac{\bbE[U]}{\bbE[V]}\leq \frac{U}{V}\leq \bigg(\frac{1+\eps}{1-\eps}\bigg)\frac{\bbE[U]}{\bbE[V]} \label{A1}. 
	\end{align}
	From \eqref{A1}, with probability at least $1-\alpha$, we have
	\begin{align}
	\bigg|\frac{U}{V}-\frac{\bbE[U]}{\bbE[V]}\bigg|\leq \frac{2\eps}{1-\eps}\bigg(\frac{\bbE[U]}{\bbE[V]}\bigg) \label{A2}. 
	\end{align}
	It follows from \eqref{A2} that
	\begin{align}
	\bbE\bigg[\bigg|\frac{U}{V}-\frac{\bbE[U]}{\bbE[V]}\bigg|^2\bigg]&=\bbE\bigg[\bigg|\frac{U}{V}-\frac{\bbE[U]}{\bbE[V]}\bigg|^2\bigg|D\bigg]\bbP(D)+ \bbE\bigg[\bigg|\frac{U}{V}-\frac{\bbE[U]}{\bbE[V]}\bigg|^2\bigg|D^c\bigg]\bbP(D^c)\\
	&\leq \frac{4\eps^2}{(1-\eps)^2}\bigg(\frac{\bbE[U]}{\bbE[V]}\bigg)^2+ \bigg[\max\bigg(M, \frac{\bbE[U]}{\bbE[V]}\bigg)\bigg]^2\alpha. 
	\end{align}
\section{Proof of Lemma \ref{aux:lem1}}\label{aux:lem1proof}
The transition kernel for the sequence $\{\gamma^{(t)}\}$ can be written as
	\begin{align}
	K^*(\gamma \to \gamma')=\frac{S}{P}\sum_{j=1}^P f(j|\gamma)\delta(\gamma'-\texttt{flip}(\gamma|j))+ \bigg(1-\frac{S}{P}\bigg)\delta(\gamma'-\gamma). 
	\end{align}
This implies that for any pair $(\gamma, \gamma')$ such that $\gamma'= \texttt{flip}(\gamma|i)$ for some $i \in [P]$, we have
	\begin{align}
	K^*(\gamma \to \gamma')&=\frac{S}{P}\sum_{j=1}^P f(j|\gamma)\delta(\gamma'-\texttt{flip}(\gamma|j))\\
	&= \frac{S}{P} f(i|\gamma).
	\end{align}
	Now, by ALG \ref{ALG1}, we also have
	\begin{align}
	f(i|\gamma)&=\phi^{-1}(\gamma) \frac{\frac{1}{2}\eta(\gamma_{-i})}{p(\gamma_i|\gamma_{-i},\calD)} \label{mat1}
	\end{align}
	and
	\begin{align}
   f(i|\gamma')&=\phi^{-1}(\gamma') \frac{\frac{1}{2}\eta(\gamma'_{-i})}{p(\gamma'_i|\gamma'_{-i},\calD)} \label{mat2}.
	\end{align}
From \eqref{mat1} and \eqref{mat2} and $\gamma_{-i}=\gamma'_{-i}$, we obtain
	\begin{align}
	\frac{K^*(\gamma \to \gamma')}{K^*(\gamma'\to \gamma)}&=\frac{\frac{S}{P}f(i|\gamma)}{\frac{S}{P}f(i|\gamma')}\\
	&=\frac{f(i|\gamma)}{f(i|\gamma')}\\
	&= \frac{\phi(\gamma') p(\gamma'|\calD)}{\phi(\gamma) p(\gamma|\calD)}\\
	&=\frac{f(\gamma')}{f(\gamma)} \label{amo0}. 
	\end{align}
In addition, we also have $K^*(\gamma \to \gamma')=K^*(\gamma'\to \gamma)=0$ if $\gamma'\neq \gamma$ and $\gamma'\neq \texttt{flip}(\gamma|i)$ for any $i \in [P]$. Furthermore, $K^*(\gamma \to \gamma')=K^*(\gamma'\to \gamma)=1-\frac{S}{P}$ if $\gamma=\gamma'$. 

By combining all these cases, it holds that
	\begin{align}
	f(\gamma) K^*(\gamma \to \gamma')=f(\gamma') K^*(\gamma'\to \gamma) \label{amo}
	\end{align} for all $\gamma',\gamma$.
	
This means that $\{\gamma^{(t)}\}_{t=1}^T$ form a reversible Markov chain with stationary distribution $f(\gamma)/Z_f$ where
	\begin{align}
	Z_f=\sum_{\gamma} f(\gamma). 
	\end{align}
Since $\{Q_t\}_{t=1}^T$ is an i.i.d. Bernoulli sequence with $q(1)=S/P$ and independent of $\{\gamma^{(t)}\}_{t=1}^T$, 
	$\{\gamma^{(t)},Q^{(t)}\}_{t=1}^T$ forms a Markov chain with the transition kernel satisfying:
	\begin{align}
	K((\gamma,Q)\to (\gamma',Q'))=q(Q')K^*(\gamma \to \gamma') \label{lubu}.
	\end{align}
It follows from \eqref{lubu} that
	\begin{align}
	q(Q) f(\gamma)/Z_f K((\gamma,Q)\to (\gamma',Q'))=\big[K^*(\gamma \to \gamma')f(\gamma)/Z_f\big] q(Q) q(Q') \label{lubu2}
	\end{align} for any pair $(\gamma,Q)$ and $(\gamma',Q')$. 

Finally, from \eqref{amo} and \eqref{lubu2}, we have
	\begin{align}
	q(Q) f(\gamma)/Z_f K((\gamma,Q)\to (\gamma',Q'))=q(Q') f(\gamma)/Z_f K((\gamma',Q')\to (\gamma,Q)).
	\end{align}  
This means that $\{\gamma_t,Q^{(t)}\}_{t=1}^T$ forms a reversible Markov chain with stationary distribution $q(Q) f(\gamma)/Z_f$.
\section{Proof of Lemma \ref{lem2}} \label{lem2:proof}
First, by definition of $\hat{\phi}(\gamma)$ in \eqref{defhphi} we have
	\begin{align}
	\rho^{(t)}= \frac{ \hat{\phi}(\gamma^{(t)})}{\sum_{t=1}^T\hat{\phi}\big(\gamma^{(t)}\big)}.
	\end{align}	
In addition, observe that
	\begin{align}
	0\leq \hat{\phi}(\gamma)\leq 1.
	\end{align}
Now, let $g: \{0,1\}^P \to \bbR_+$ such that $g(\gamma) \leq 1$ for all $\gamma$. Then, by applying Lemma \ref{hoeffinglem} and a change of measure, with probability $1-2\frac{d\nu}{d\pi}\exp(-\frac{\zeta^2 T(1-\lambda)}{64 e})$, we have
	\begin{align}
	\frac{1}{T}\bigg|\sum_{t=1}^T\hat{\phi}(\gamma^{(t)}) g(\gamma^{(t)})Q^{(t)}-\bbE_{\pi}\bigg[\sum_{t=1}^T\hat{\phi}(\gamma^{(t)}) g(\gamma^{(t)})Q^{(t)}\bigg]\bigg|\leq \zeta \label{B2mod}
	\end{align} for any $\zeta>0$. 
	
	Similarly, by using Lemma \ref{hoeffinglem}, with probability at least $1-2\frac{d\nu}{d\pi}\exp(-\frac{\zeta^2 T(1-\lambda)}{64 e})$, it holds that
	\begin{align}
	\frac{1}{T}\bigg|\sum_{t=1}^T \hat{\phi}(\gamma^{(t)})Q^{(t)}-\bbE_{\pi}\bigg[\sum_{t=1}^T \hat{\phi}(\gamma^{(t)})Q^{(t)}\bigg]\bigg|\leq \zeta \label{B4mod}. 
	\end{align}
By using the union bound, with probability at least $1-4\frac{d\nu}{d\pi}\exp(-\frac{\zeta^2 T(1-\lambda)}{64 e})$, it holds that
	\begin{align}
	\frac{1}{T}\bigg|\sum_{t=1}^T\hat{\phi}(\gamma^{(t)}) g(\gamma^{(t)})Q^{(t)}-\bbE_{\pi}\bigg[\sum_{t=1}^T\hat{\phi}(\gamma^{(t)}) g(\gamma^{(t)})Q^{(t)}\bigg]\bigg|\leq \zeta,\\
	\frac{1}{T}\bigg|\sum_{t=1}^T \hat{\phi}(\gamma^{(t)})-\bbE_{\pi}\bigg[\sum_{t=1}^T \hat{\phi}(\gamma^{(t)})\bigg]\bigg|\leq \zeta \label{U6mod}.
	\end{align}
Now, by setting $\zeta=\zeta_0:= \frac{\eps}{T}\min\big\{\bbE_{\pi}\big[\sum_{t=1}^T\hat{\phi}(\gamma^{(t)}) g(\gamma^{(t)})Q^{(t)}\big],\bbE_{\pi}\big[\sum_{t=1}^T \hat{\phi}(\gamma^{(t)})\big] \big\}$ for some $\eps>0$ (to be chosen later), with probability at least $1-4\frac{d\nu}{d\pi}\exp(-\frac{\zeta_0^2 T(1-\lambda)}{64 e})$, it holds that
	\begin{align}
	&\frac{1}{T}\bigg|\sum_{t=1}^T\hat{\phi}(\gamma^{(t)}) g(\gamma^{(t)})Q^{(t)}-\bbE_{\pi}\bigg[\sum_{t=1}^T\hat{\phi}(\gamma^{(t)}) g(\gamma^{(t)})Q^{(t)}\bigg]\bigg|\nn\\
	&\qquad \leq \frac{\eps}{T} \bbE_{\pi}\bigg[\sum_{t=1}^T\hat{\phi}(\gamma^{(t)}) g(\gamma^{(t)})Q^{(t)}\bigg],\\
	&\frac{1}{T}\bigg|\sum_{t=1}^T \hat{\phi}(\gamma^{(t)})Q^{(t)}-\bbE_{\pi}\bigg[\sum_{t=1}^T \hat{\phi}(\gamma^{(t)})Q^{(t)}\bigg]\bigg|\leq \frac{\eps}{T}\bbE_{\pi}\bigg[\sum_{t=1}^T \hat{\phi}(\gamma^{(t)})Q^{(t)}\bigg]  \label{U7mod}.
	\end{align}
Furthermore, by setting
	\begin{align} U&:=\frac{1}{T}\sum_{t=1}^T\hat{\phi}(\gamma^{(t)}) g(\gamma^{(t)})Q^{(t)},\\
	V&:=\frac{1}{T}\sum_{t=1}^T \hat{\phi}(\gamma^{(t)})Q^{(t)},
	\end{align}
	we have
	\begin{align}
	\frac{U}{V}&=\frac{ \sum_{t=1}^T\phi^{-1}(\gamma^{(t)}) g(\gamma^{(t)})Q^{(t)}}{\sum_{t=1}^T \phi^{-1}(\gamma^{(t)})Q^{(t)}}\\
	&=\sum_{t=1}^T \rho^{(t)} g(\gamma^{(t)})
	\end{align}
	and
	\begin{align}
	M:=\sup(U/V)\leq 1 \label{eqU6ma}
	\end{align} since $\sum_{t=1}^T \rho^{(t)}=1$ and $g(\gamma^{(t)})\leq 1$ for all $\gamma^{(t)}$.
	
From \eqref{U6mod}-\eqref{eqU6ma}, by Lemma \ref{lem0}, we have
	\begin{align}
	\bbE\bigg[\bigg|\sum_{t=1}^T\rho^{(t)} g(\gamma^{(t)})Q^{(t)}-\frac{\bbE_{\pi}[U]}{\bbE_{\pi}[V]}\bigg|^2\bigg]\leq  \frac{4\eps^2}{(1-\eps)^2}\bigg(\frac{\bbE_{\pi}[U]}{\bbE_{\pi}[V]}\bigg)^2+ \bigg[\max\bigg(1, \frac{\bbE_{\pi}[U]}{\bbE_{\pi}[V]}\bigg)\bigg]^2\alpha,
	\end{align} where $\alpha:=4\frac{d\nu}{d\pi}\exp\big(-\frac{\eps^2 T(1-\lambda_{\gamma,Q})\min\{\bbE_{\pi}[U],\bbE_{\pi}[V]\}^2}{64 e}\big)$, 
where $\lambda_{\gamma,Q}$ is the stationary distribution of the reversible Markov chain $\{\gamma^{(t)},Q^{(t)}\}$. 

Now, by setting
	\begin{align}
	\eps=\eps_0&=\frac{1}{\min\{\bbE_{\pi}[U],\bbE_{\pi}[V]\}} \sqrt{\frac{64 e \log T}{(1-\lambda_{\gamma,Q})T}},
	\end{align} we have $\alpha=4 \frac{d\nu}{d\pi} \frac{1}{T}$. Then, we obtain
	\begin{align}
	\bbE\bigg[\bigg|\sum_{t=1}^T\rho^{(t)} g(\gamma^{(t)})-\frac{\bbE_{\pi}[U]}{\bbE_{\pi}[V]}\bigg|^2\bigg]\leq \frac{4 \eps_0^2}{(1-\eps_0)^2}\bigg(\frac{\bbE_{\pi}[U]}{\bbE_{\pi}[V]}\bigg)^2+ \bigg[\max\bigg(1, \frac{\bbE_{\pi}[U]}{\bbE_{\pi}[V]}\bigg)\bigg]^2\alpha \label{C1mod}.
	\end{align}
Now, observe that
	\begin{align}
	\frac{\bbE_{\pi}[U]}{\bbE_{\pi}[V]}&=\frac{\bbE_{\pi}\big[g(\gamma)Q\hat{\phi}(\gamma)\big]}{\bbE_{\pi}\big[\hat{\phi}(\gamma)Q\big]}\\
	&=\frac{\bbE_{\pi}\big[g(\gamma)Q\phi^{-1}(\gamma)\big]}{\bbE_{\pi}\big[\phi^{-1}(\gamma)Q\big]} \label{E1mod}.
	\end{align}
On the other hand, by Lemma \ref{aux:lem1}, we have $\pi(\gamma,Q)= \frac{q(Q) f(\gamma)}{Z_f}$ where $Z_f:=\sum_{\gamma} f(\gamma)$ and $f(\gamma)=p(\gamma|\calD) \phi(\gamma)$.  It follows that
	\begin{align}
	\bbE_{\pi}\big[g(\gamma)Q\phi^{-1}(\gamma)\big]&=\bbE_{q(Q)f(\gamma)/Z_f}\big[g(\gamma)Q\phi^{-1}(\gamma)\big]\\
	&=\sum_{\gamma}\sum_Q g(\gamma) Q\phi^{-1}(\gamma) \frac{f(\gamma)}{Z_f} q(Q)\\
	&=\frac{1}{Z_f}\sum_{\gamma}\sum_Q g(\gamma) q(Q) Q p(\gamma|\calD) \\
	&=\frac{1}{Z_f}\bbE_{p(\gamma|\calD)}\big[g(\gamma) \big] \bbE_q[Q]
	\label{E2mod}.
	\end{align}
	Similarly, we have
	\begin{align}
	\bbE_{\pi}\big[\phi^{-1}(\gamma)Q\big]&=\bbE_{q(Q)f(\gamma)/Z_f}\big[\phi^{-1}(\gamma)Q\big]\\
	&=\sum_Q \sum_{\gamma}\phi^{-1}(\gamma)Q  \frac{f(\gamma)}{Z_f}q(Q) \\
	&= \frac{1}{Z_f}\bigg(\sum_{\gamma} P(\gamma|\calD) \bigg) \bbE_q[Q]\label{E3mod}.
	\end{align}
	From \eqref{E1mod}, \eqref{E2mod} and \eqref{E3mod}, we obtain
	\begin{align}
	\frac{\bbE_{\pi}[U]}{\bbE_{\pi}[V]}=\bbE_{p(\gamma|\calD)}\big[g(\gamma)\big] \label{xmo}. 
	\end{align}

For the given problem, by setting $g(\gamma)=p(\gamma_i=1| \gamma_{-i},\calD)$, from \eqref{xmo}, we have
\begin{align}
\frac{\bbE_{\pi}[U]}{\bbE_{\pi}[V]}= \texttt{PIP}(i) \label{eq68}.  
\end{align}
In addition, we have
\begin{align}
\bbE_{\pi}[V]&= \bbE_{\pi}\big[\hat{\phi}(\gamma)Q\big]\\
&=\sum_{\gamma,Q} \hat{\phi}(\gamma)Q \frac{f(\gamma)}{Z_f} q(Q)\\
&=\bigg(\sum_{\gamma} \hat{\phi}(\gamma)\frac{f(\gamma)}{Z_f}\bigg)\bigg(\sum_Q Q q(Q)\bigg)\\
&=\bbE_{\pi}[\hat{\phi}(\gamma)] \bbE_Q[Q]\\
&= \frac{S}{P}\bbE_{\pi}[\hat{\phi}(\gamma)] \label{hg1}. 
\end{align}
Hence, we obtain
\begin{align}
\min\{\bbE_{\pi}[U],\bbE_{\pi}[V] \}&=\bbE_{\pi}[V]\min\bigg\{1,\frac{\bbE_{\pi}[U]}{\bbE_{\pi}[V]} \bigg\}\\
&=\bbE_{\pi}[V]\min\bigg\{1, \texttt{PIP}(i) \bigg\}\\
&= \bbE_{\pi}[V]\texttt{PIP}(i) \\
&= \frac{S}{P}\bbE_{\pi}[\hat{\phi}(\gamma)]\texttt{PIP}(i) \label{eq69}.
\end{align}
From \eqref{C1mod}, \eqref{eq68}, and \eqref{eq69}, we have
\begin{align}
\bbE\bigg[\bigg|\sum_{t=1}^T\rho^{(t)} p(\gamma_i^{(t)}=1| \gamma_{-i}^{(t)},\calD)-\texttt{PIP}(i) \bigg|^2\bigg] \leq   \frac{4\eps_0^2 }{(1-\eps_0)^2}\texttt{PIP}^2(i) + 4\frac{d\nu}{d\pi}\frac{1}{T} \label{eq98},
\end{align}
and
\begin{align}
\eps_0=\frac{P}{\texttt{PIP}(i)\bbE_{\pi}[\hat{\phi}(\gamma)]S} \sqrt{\frac{64 e \log T}{(1-\lambda_{\gamma,Q})T}} \label{defeps0}.
\end{align}
Now, observe that
\begin{align}
\frac{d\nu}{d\pi}(\gamma,Q)&= \frac{p_{\gamma_1,Q_1}(\gamma,Q)}{\pi(\gamma,Q)}\\
&\leq \frac{1}{\pi(\gamma,Q)}\\
&= \frac{1}{\pi(\gamma)q(Q)}\\
&\leq \frac{P}{S} \frac{1}{\min_{\gamma} \pi(\gamma)} \label{eq99}. 
\end{align}
By combining \eqref{eq98} and \eqref{eq99}, we have
\begin{align}
\bbE\bigg[\bigg|\sum_{t=1}^T\rho^{(t)} p(\gamma_i^{(t)}=1| \gamma_{-i}^{(t)},\calD)-\texttt{PIP}(i) \bigg|^2\bigg] \leq \frac{4\eps_0^2}{(1-\eps_0)^2}\texttt{PIP}^2(i) + \frac{4P}{S} \frac{1}{\min_{\gamma} \pi(\gamma)T} \label{smes}.
\end{align}  
 
\section{Derive $p(\gamma_i|\calD, \gamma_{-i})$} \label{posterapp}
 Observe that
\begin{align}
p(\gamma_i|\calD,\gamma_{-i})&= \frac{p(\gamma_i|\calD,\gamma_{-i})}{p(1-\gamma_i|\calD,\gamma_{-i})}\bigg(1+ \frac{p(\gamma_i|\calD,\gamma_{-i})}{p(1-\gamma_i|\calD,\gamma_{-i})}\bigg)^{-1}.
\end{align}
In addition, we have
\begin{align}
\frac{p(\gamma_i=1|\calD,\gamma_{-i})}{p(\gamma_i=0|\calD,\gamma_{-i})}&= \frac{p(\gamma_i=1,\calD|\gamma_{-i})}{p(\gamma_i=0,\calD|\gamma_{-i})}\\
&= \frac{p(\gamma_i=1|\gamma_{-i},X)}{p(\gamma_i=0|\gamma_{-i},X)}\frac{p(Y|\gamma_i=1,\gamma_{-i},X)}{p(Y|\gamma_i=0,\gamma_{-i},X)}\\
&= \bigg(\frac{p(\gamma_i=1)}{p(\gamma_i=0)}\bigg)\bigg(\frac{p(Y|\gamma_i=1,\gamma_{-i},X)}{p(Y|\gamma_i=0,\gamma_{-i},X)}\bigg)\\
&= \bigg(\frac{h}{1-h}\bigg)\bigg(\frac{p(Y|\gamma_i=1,\gamma_{-i},X)}{p(Y|\gamma_i=0,\gamma_{-i},X)}\bigg) \label{C1}. 
\end{align}
On the other hand, for any tuple $\gamma=(\gamma_1,\gamma_2,\cdots,\gamma_P)$ such that $\gamma_i=1$ (so $|\gamma|\geq 1$), we have
\begin{align}
p(Y|\gamma_i=1,\gamma_{-i},\beta_{\gamma},\sigma_{\gamma}^2,X)&=\frac{1}{\big(\sigma_{\gamma} \sqrt{2\pi}\big)^N}\exp\bigg(-\frac{\|Y-X_{\gamma}\beta_{\gamma}\|^2}{2\sigma_{\gamma}^2}\bigg) \label{C3}.
\end{align}
It follows that
\begin{align}
&p(Y|\gamma_i=1,\gamma_{-i},X\big)\nn\\
& = \int_{\beta_{\gamma}} \int_{\sigma_{\gamma}^2=0}^{\infty} \frac{1}{\big(\sigma_{\gamma} \sqrt{2\pi}\big)^N}\exp\bigg(-\frac{\|Y-X_{\gamma}\beta_{\gamma}\|^2}{2\sigma_{\gamma}^2}\bigg)p(\beta_{\gamma}|\gamma_i=1,\gamma_{-i})p(\sigma_{\gamma}^2|\gamma_i=1,\gamma_{-i}) d\beta_{\gamma} d\sigma_{\gamma}^2\\
&= \int_{\sigma_{\gamma}^2=0}^{\infty}\texttt{InvGamma}\bigg(\frac{1}{2}\nu_0, \frac{1}{2}\nu_0 \lambda_0\bigg)\int_{\beta_{\gamma}} \frac{1}{\big(\sigma_{\gamma} \sqrt{2\pi}\big)^N}\exp\bigg(-\frac{\|Y-X_{\gamma}\beta_{\gamma}\|^2}{2\sigma_{\gamma}^2}\bigg)\nn\\
&\qquad   \times  \frac{1}{\big(\sigma_{\gamma}\sqrt{2\pi \tau^{-1}}\big)^{|\gamma|}}\exp\bigg(-\frac{\|\beta_{\gamma}\|^2}{2 \sigma_{\gamma}^2 \tau^{-1}}\bigg)
d\beta_{\gamma} d\sigma_{\gamma}^2 \label{M00}.
\end{align}
Now, observe that
\begin{align}
&\|Y-X_{\gamma} \beta_{\gamma}\|^2+ \tau \|\beta_{\gamma}\|^2\nn\\
&\qquad =(Y-X_{\gamma}\beta_{\gamma})^T(Y-X_{\gamma}\beta_{\gamma}) + \tau \beta_{\gamma}^T \beta_{\gamma}\\
&\qquad = Y^T Y -2 Y^T X_{\gamma} \beta_{\gamma} + \beta_{\gamma}^T X_{\gamma}^T X_{\gamma} \beta_{\gamma} + \tau \beta_{\gamma}^T \beta_{\gamma}\\
&\qquad = Y^T Y -2 Y^T X_{\gamma} \beta_{\gamma}+ \beta_{\gamma}^T (X_{\gamma}^T X_{\gamma} +\tau I)\beta_{\gamma} \label{D1}. 
\end{align}
Now, consider the EVD (singular value decomposition) of the positive definite matrix $X_{\gamma}^T X_{\gamma} +\tau I$ (note that $\tau>0$):
\begin{align}
X_{\gamma}^T X_{\gamma}+ \tau I= U^T \Lambda U 
\end{align} where $\Lambda$ is the a diagonal matrix consisting of all positive eigenvalue of $X_{\gamma}^T X_{\gamma}+ \tau I$. Let
\begin{align}
\tbeta_{\gamma}&:= \sqrt{\Lambda} U \beta_{\gamma},\\
\tilY_{\gamma}&:=\sqrt{\Lambda^{-1}}U  X_{\gamma}^T Y.  
\end{align}
Then, we have
\begin{align}
&\|Y-X_{\gamma} \beta_{\gamma}\|^2+ \tau \|\beta_{\gamma}\|^2\nn\\
&\qquad= Y^T Y -2 Y^T X_{\gamma} \beta_{\gamma}+ \beta_{\gamma}^T (X_{\gamma}^T X_{\gamma} +\tau I)\beta_{\gamma}\\
&\qquad=  Y^T Y -2 Y^T X_{\gamma} \sqrt{\Lambda^{-1}}U^T\tbeta_{\gamma}+ \tbeta_{\gamma}^T \tbeta_{\gamma}\\
&\qquad=  Y^T Y -2 \tilY_{\gamma}^T \tbeta_{\gamma}+ \tbeta_{\gamma}^T \tbeta_{\gamma}\\
&\qquad= \big( \|Y\|^2-\|\tilY_{\gamma}|^2\big)+ \big(\tilY_{\gamma}^T \tilY_{\gamma} -2 \tilY_{\gamma}^T \tbeta_{\gamma}+ \tbeta_{\gamma}^T \tbeta_{\gamma}\big)\\
&\qquad=  \big( \|Y\|^2-\|\tilY_{\gamma}|^2\big)+ \|\tilY_{\gamma}-\tbeta_{\gamma}\|^2 \label{M1}.
\end{align}
Hence, we have
\begin{align}
d\beta_{\gamma}&= \det(U^T \Lambda^{-1/2})  d\tbeta_{\gamma}\\
&= \det(X_{\gamma}^T X_{\gamma} +\tau I)^{-1/2}d\tbeta_{\gamma} \label{M2}.
\end{align}
Hence, we have
\begin{align}
&\int_{\beta_{\gamma}} \frac{1}{\big(\sigma_{\gamma} \sqrt{2\pi}\big)^N}\exp\bigg(-\frac{\|Y-X_{\gamma}\beta_{\gamma}\|^2}{2\sigma_{\gamma}^2}\bigg)  \frac{1}{\big(\sigma_{\gamma}\sqrt{2\pi \tau^{-1}}\big)^{|\gamma|}}\exp\bigg(-\frac{\|\beta_{\gamma}\|^2}{2 \sigma_{\gamma}^2 \tau^{-1}}\bigg)
d\beta_{\gamma}\\
&\qquad =\int_{\tbeta_{\gamma}} \frac{1}{\big(\sigma_{\gamma} \sqrt{2\pi}\big)^N}\exp\bigg(-\frac{ \big( \|Y\|^2-\|\tilY_{\gamma}|^2\big)+ \|\tilY_{\gamma}-\tbeta_{\gamma}\|^2}{2\sigma_{\gamma}^2}\bigg)  \nn\\
&\qquad\qquad  \times \frac{1}{\big(\sigma_{\gamma}\sqrt{2\pi \tau^{-1}}\big)^{|\gamma|}}
\det(X_{\gamma}^T X_{\gamma} +\tau I)^{-1/2}d\tbeta_{\gamma}\\
&\qquad= \frac{1}{\big(\sigma_{\gamma} \sqrt{2\pi}\big)^N} \tau^{|\gamma|/2}\exp\bigg(-\frac{ \big( \|Y\|^2-\|\tilY_{\gamma}|^2\big)}{2\sigma_{\gamma}^2}\bigg) \det(X_{\gamma}^T X_{\gamma} +\tau I)^{-1/2} \label{M4}.
\end{align}
By combining \eqref{M00} and \eqref{M4}, we obtain
\begin{align}
&p(Y|\gamma_i=1,\gamma_{-i},X\big)\nn\\
& = \int_{\beta_{\gamma}} \int_{\sigma_{\gamma}^2=0}^{\infty} \frac{1}{\big(\sigma_{\gamma} \sqrt{2\pi}\big)^N}\exp\bigg(-\frac{\|Y-X_{\gamma}\beta_{\gamma}\|^2}{2\sigma_{\gamma}^2}\bigg)p(\beta_{\gamma}|\gamma_i=1,\gamma_{-i})p(\sigma_{\gamma}^2|\gamma_i=1,\gamma_{-i}) d\beta_{\gamma} d\sigma_{\gamma}^2\\
&= \int_{\sigma_{\gamma}^2=0}^{\infty}\texttt{InvGamma}\bigg(\frac{1}{2}\nu_0, \frac{1}{2}\nu_0 \lambda_0\bigg)\frac{1}{\big(\sigma_{\gamma} \sqrt{2\pi}\big)^N} \tau^{|\gamma|/2}\nn\\
&\qquad \times \exp\bigg(-\frac{ \big( \|Y\|^2-\|\tilY_{\gamma}|^2\big)}{2\sigma_{\gamma}^2}\bigg) \det(X_{\gamma}^T X_{\gamma} +\tau I)^{-1/2} d\sigma_{\gamma}^2 \label{M0}\\
&\qquad= \det(X_{\gamma}^T X_{\gamma} +\tau I)^{-1/2}\tau^{|\gamma|/2} (2\pi)^{-N/2}\int_{\sigma_{\gamma}^2=0}^{\infty}\texttt{InvGamma}\bigg(\frac{1}{2}\nu_0, \frac{1}{2}\nu_0 \lambda_0\bigg) (\sigma_{\gamma}^2)^{-N/2} \nn\\
&\qquad \qquad \times \exp\bigg(-\frac{ \big( \|Y\|^2-\|\tilY_{\gamma}\|^2\big)}{2\sigma_{\gamma}^2}\bigg)  d\sigma_{\gamma}^2\\
&\qquad= \det(X_{\gamma}^T X_{\gamma} +\tau I)^{-1/2}\tau^{|\gamma|/2} (2\pi)^{-N/2}\nn\\
&\qquad \times \int_{\sigma_{\gamma}^2=0}^{\infty}
\frac{(1/2 \lambda_0 \nu_0)^{1/2\nu_0}}{\Gamma(1/2 \nu_0)} (1/\sigma_{\gamma}^2)^{1/2 \nu_0+1}\exp\bigg(-1/2 \nu_0 \lambda_0/\sigma_{\gamma}^2\bigg)
(\sigma_{\gamma}^2)^{-N/2} \nn\\
&\qquad \qquad \times \exp\bigg(-\frac{ \big( \|Y\|^2-\|\tilY_{\gamma}\|^2\big)}{2\sigma_{\gamma}^2}\bigg)  d\sigma_{\gamma}^2\\
&\qquad= \det(X_{\gamma}^T X_{\gamma} +\tau I)^{-1/2}\tau^{|\gamma|/2} (2\pi)^{-N/2}\frac{(1/2 \lambda_0 \nu_0)^{1/2\nu_0}}{\Gamma(1/2 \nu_0)} \nn\\
&\qquad \times \int_{\sigma_{\gamma}^2=0}^{\infty}(1/\sigma_{\gamma}^2)^{1/2 \nu_0+1+N/2}
\exp\bigg(-\frac{ \big( \|Y\|^2-\|\tilY_{\gamma}\|^2+\nu_0 \lambda_0\big)}{2\sigma_{\gamma}^2}\bigg)  d\sigma_{\gamma}^2\\
&\qquad= \det(X_{\gamma}^T X_{\gamma} +\tau I)^{-1/2}\tau^{|\gamma|/2} (2\pi)^{-N/2}\frac{(1/2 \lambda_0 \nu_0)^{1/2\nu_0}}{\Gamma(1/2 \nu_0)}\nn\\
&\qquad \qquad \times \Gamma\bigg(\frac{N+\nu_0}{2}\bigg)\bigg( \frac{\|Y\|^2-\|\tilY_{\gamma}\|^2+ \nu_0 \lambda_0}{2}\bigg)^{-\frac{N+\nu_0}{2}} \label{M10}. 
\end{align}
Let $\tgamma_1$ is given by $\gamma_{-i}$ with $\gamma_i=1$, $\tgamma_0$ is given by $\gamma_{-i}$ with $\gamma_i=0$. It follows that
\begin{align}
\frac{p(Y|\gamma_i=1,\gamma_{-i},X)}{p(Y|\gamma_i=0,\gamma_{-i},X)}=\sqrt{\tau} \sqrt{\frac{\det(X_{\tgamma_0}^T X_{\tgamma_0} +\tau I)}{\det(X_{\tgamma_1}^T X_{\tgamma_1} +\tau I)}}\bigg(\frac{\|Y\|^2-\|\tilY_{\tgamma_0}\|^2+\nu_0 \lambda_0}{\|Y\|^2-\|\tilY_{\tgamma_1}\|^2+\nu_0 \lambda_0}\bigg)^{\frac{N+\nu_0}{2}} \label{finxo}.
\end{align}
On the other hand, we have
\begin{align}
\|\tilY_{\gamma}\|^2&= \tilY_{\gamma}^T \tilY_{\gamma}\\
&= Y^T X_{\gamma} (X_{\gamma}^T X_{\gamma}+\tau I)^{-1} X_{\gamma}^T Y. 
\end{align}
Hence, we finally have
\begin{align}
\frac{p(Y|\gamma_i=1,\gamma_{-i},X)}{p(Y|\gamma_i=0,\gamma_{-i},X)}=\sqrt{\tau \frac{\det(X_{\tgamma_0}^T X_{\tgamma_0} +\tau I)}{\det(X_{\tgamma_1}^T X_{\tgamma_1} +\tau I)}\bigg(\frac{S_{\tgamma_0}}{S_{\tgamma_1}}\bigg)^{N+\nu_0}} \label{eq171},
\end{align} where
\begin{align}
S_{\gamma}:= Y^T Y - Y^T X_{\gamma} (X_{\gamma}^T X_{\gamma}+\tau I)^{-1} X_{\gamma}^T Y+\nu_0 \lambda_0. 
\end{align}
Based on this, we can estimate
\begin{align}
p(\gamma_i|\calD,\gamma_{-i})&= \frac{p(\gamma_i|\calD,\gamma_{-i})}{p(1-\gamma_i|\calD,\gamma_{-i})}\bigg(1+ \frac{p(\gamma_i|\calD,\gamma_{-i})}{p(1-\gamma_i|\calD,\gamma_{-i})}\bigg)^{-1}.
\end{align}
Denote the set of included variables in $\tgamma_0$ as $I=\{j: \tgamma_{0,j}=1\}$ . Define $F=\big(X_{\tgamma_0}^T X_{\tgamma_0}+\tau I\big)^{-1}$, $\nu=X^T Y$ and $\nu_{\tgamma_0}=(\nu_j)_{j\in I}$. Also define $A=X^T X$ and $a_i=(A_{ji})_{j \in I}$. Then, by using the same arguments as  \cite[Appendix B1]{Zanella2019ScalableIT}, we can show that
\begin{align}
S(\tgamma_1)=S(\tgamma_0)-d_i\big(\nu_{\tgamma_0}^T F a_i-\nu_i\big)^2,
\end{align} where $d_i=(A_{ii}+\tau-a_i^T  F a_i)^{-1}$. In addition, we can compute $a_i^T F a_i$ by using the Cholesky decomposition of $F=LL^T$ and 
\begin{align}
a_i^T F a_i&= \|a_i^T L\|^2\\
&= \sum_{j \in I} (BL)_{ij}^2,
\end{align}  where $B$ is the $p\times |\gamma|$ matrix made of the columns of $A$ corresponding to variables included in $\gamma$. 

In addition, we have
\begin{align}
X_{\tgamma_1}^T X_{\tgamma_1} +\tau I=\begin{pmatrix}X_{\tgamma_0}^T X_{\tgamma_0}+\tau I &a_i\\a_i^T & A_{ii}+\tau \end{pmatrix}
\end{align}
Hence, by using Schur's formula for the determinant of block matrix, we are easy to see that
\begin{align}
\frac{\det(X_{\tgamma_0}^T X_{\tgamma_0} +\tau I)}{\det(X_{\tgamma_1}^T X_{\tgamma_1} +\tau I)}=d_i. 
\end{align}
Using this algorithm, if pre-computing $X^TX$ is not possible, the computational complexity per MCMC iteration is $O(N|\gamma|^2+ |\gamma|^3+P|\gamma|^2)$. Otherwise, if pre-computing $X^TX$ is possible, the computational complexity per MCMC iteration is $O(|\gamma|^3+P|\gamma|^2)$.
\end{appendices}
\bibliography{sn-bibliography}

\input MCMCv15.bbl
\end{document}

%% file: preamble.tex
\usepackage[mathscr]{eucal}
\usepackage[cmex10]{amsmath}
\usepackage{epsfig,epsf,psfrag}
\usepackage{amssymb,amsmath,amsthm,amsfonts,latexsym,bm}
\usepackage{amsmath,graphicx,bm,xcolor,url}
\usepackage{array}
\usepackage{verbatim}
\usepackage{bm}
\usepackage{verbatim}
\usepackage{textcomp}
\usepackage{mathrsfs}
\usepackage{multirow}
\usepackage{epstopdf}

\catcode`~=11 \def\UrlSpecials{\do\~{\kern -.15em\lower .7ex\hbox{~}\kern .04em}} \catcode`~=13 

\allowdisplaybreaks[1]
 
\newcommand{\nn}{\nonumber}

\newcommand{\calD}{\mathcal{D}}
\newcommand{\calN}{\mathcal{N}}

\newcommand{\calU}{\mathcal{U}}

\newcommand{\bI}{\mathbf{I}}

\newcommand{\bK}{\mathbf{K}}

\newcommand{\bbC}{\mathbb{C}}

\DeclareMathAlphabet{\mathbsf}{OT1}{cmss}{bx}{n}
\DeclareMathAlphabet{\mathssf}{OT1}{cmss}{m}{sl}

\DeclareSymbolFont{bsfletters}{OT1}{cmss}{bx}{n}  
\DeclareSymbolFont{ssfletters}{OT1}{cmss}{m}{n}
\DeclareMathSymbol{\bsfGamma}{0}{bsfletters}{'000}
\DeclareMathSymbol{\ssfGamma}{0}{ssfletters}{'000}
\DeclareMathSymbol{\bsfDelta}{0}{bsfletters}{'001}
\DeclareMathSymbol{\ssfDelta}{0}{ssfletters}{'001}
\DeclareMathSymbol{\bsfTheta}{0}{bsfletters}{'002}
\DeclareMathSymbol{\ssfTheta}{0}{ssfletters}{'002}
\DeclareMathSymbol{\bsfLambda}{0}{bsfletters}{'003}
\DeclareMathSymbol{\ssfLambda}{0}{ssfletters}{'003}
\DeclareMathSymbol{\bsfXi}{0}{bsfletters}{'004}
\DeclareMathSymbol{\ssfXi}{0}{ssfletters}{'004}
\DeclareMathSymbol{\bsfPi}{0}{bsfletters}{'005}
\DeclareMathSymbol{\ssfPi}{0}{ssfletters}{'005}
\DeclareMathSymbol{\bsfSigma}{0}{bsfletters}{'006}
\DeclareMathSymbol{\ssfSigma}{0}{ssfletters}{'006}
\DeclareMathSymbol{\bsfUpsilon}{0}{bsfletters}{'007}
\DeclareMathSymbol{\ssfUpsilon}{0}{ssfletters}{'007}
\DeclareMathSymbol{\bsfPhi}{0}{bsfletters}{'010}
\DeclareMathSymbol{\ssfPhi}{0}{ssfletters}{'010}
\DeclareMathSymbol{\bsfPsi}{0}{bsfletters}{'011}
\DeclareMathSymbol{\ssfPsi}{0}{ssfletters}{'011}
\DeclareMathSymbol{\bsfOmega}{0}{bsfletters}{'012}
\DeclareMathSymbol{\ssfOmega}{0}{ssfletters}{'012}

\newcommand{\tilY}{\tilde{Y}}

\newcommand{\tbeta}{\tilde{\beta}}
\newcommand{\tgamma}{\tilde{\gamma}}

\newcommand{\eps}{\varepsilon}

\newcommand{\bone}{\mathbf{1}}

\newcommand{\qednew}{\nobreak \ifvmode \relax \else
      \ifdim\lastskip<1.5em \hskip-\lastskip
      \hskip1.5em plus0em minus0.5em \fi \nobreak
      \vrule height0.75em width0.5em depth0.25em\fi}

%% file: MCMCv15submission.bbl
\begin{thebibliography}{23}
\providecommand{\natexlab}[1]{#1}
\providecommand{\url}[1]{\texttt{#1}}
\expandafter\ifx\csname urlstyle\endcsname\relax
  \providecommand{\doi}[1]{doi: #1}\else
  \providecommand{\doi}{doi: \begingroup \urlstyle{rm}\Url}\fi

\bibitem[Andrieu et~al.(2004)Andrieu, de~Freitas, Doucet, and
  Jordan]{Andrieu2004AnIT}
C.~Andrieu, N.~de~Freitas, A.~Doucet, and M.~I. Jordan.
\newblock An introduction to {MCMC} for machine learning.
\newblock \emph{Machine Learning}, 50:\penalty0 5--43, 2004.

\bibitem[Bishop(2006)]{Bishop}
C.~M. Bishop.
\newblock \emph{Pattern Recognition and Machine Learning}.
\newblock Springer, 2006.

\bibitem[Bolstad(2010)]{Bolstad}
W.~M. Bolstad.
\newblock \emph{Understanding Computational Bayesian Statistics}.
\newblock John Wiley, 2010.

\bibitem[Breiman(1960)]{Breiman1960TheSL}
L.~Breiman.
\newblock The strong law of large numbers for a class of {Markov} chains.
\newblock \emph{Annals of Mathematical Statistics}, 31:\penalty0 801--803,
  1960.

\bibitem[Bugallo et~al.(2007)Bugallo, Xu, and
  Djuri{\'c}]{Bugallo2007PerformanceCO}
M.~F. Bugallo, S.~Xu, and P.~M. Djuri{\'c}.
\newblock Performance comparison of {EKF} and particle filtering methods for
  maneuvering targets.
\newblock \emph{Digit. Signal Process.}, 17:\penalty0 774--786, 2007.

\bibitem[Combes and Touati(2019)]{Combes2019EE}
R.~Combes and M.~Touati.
\newblock Computationally efficient estimation of the spectral gap of a markov
  chain.
\newblock \emph{Proceedings of the ACM on Measurement and Analysis of Computing
  Systems}, 3:\penalty0 1 -- 21, 2019.

\bibitem[Diaconis and Saloff-Coste(1993)]{Diaconis1993a}
P.~Diaconis and L.~Saloff-Coste.
\newblock Comparison theorems for reversible markov chains.
\newblock \emph{Annals of Applied Probability}, 3:\penalty0 696--730, 1993.

\bibitem[Fitzgerald(2001)]{Fitzgerald2001MarkovCM}
W.~J. Fitzgerald.
\newblock Markov chain {Monte Carlo} methods with applications to signal
  processing.
\newblock \emph{Signal Process.}, 81:\penalty0 3--18, 2001.

\bibitem[Gupta and Rawlings(2014)]{Gupta2014ComparisonOP}
A.~Gupta and J.~B. Rawlings.
\newblock Comparison of parameter estimation methods in stochastic chemical
  kinetic models: Examples in systems biology.
\newblock \emph{AIChE journal. American Institute of Chemical Engineers}, 60
  4:\penalty0 1253--1268, 2014.

\bibitem[Hesterberg(2002)]{Hesterberg2002MonteCS}
T.~Hesterberg.
\newblock Monte carlo strategies in scientific computing.
\newblock \emph{Technometrics}, 44:\penalty0 403 -- 404, 2002.

\bibitem[Jankowiak(2022)]{Jankowiak2022BayesianVS}
M.~Jankowiak.
\newblock Bayesian variable selection in a million dimensions.
\newblock \emph{ArXiv}, abs/2208.01180, 2022.

\bibitem[Kasim et~al.(2019)Kasim, Bott, Tzeferacos, Lamb, Gregori, and
  Vinko]{Kasim2019RetrievingFF}
M.~F. Kasim, A.~F.~A. Bott, P.~Tzeferacos, D.~Q. Lamb, G.~Gregori, and S.~M.
  Vinko.
\newblock Retrieving fields from proton radiography without source profiles.
\newblock \emph{Physical review. E}, 100 3-1:\penalty0 033208, 2019.

\bibitem[Liang et~al.(2010)Liang, Liu, and Carroll]{Liang2010AdvancedMC}
F.~Liang, C.~Liu, and R.~J. Carroll.
\newblock Advanced {Markov} chain {Monte Carlo} methods: Learning from past
  samples.
\newblock 2010.

\bibitem[Paulin(2015)]{Daniel2015}
D.~Paulin.
\newblock {Concentration inequalities for Markov chains by Marton couplings and
  spectral methods}.
\newblock \emph{Electronic Journal of Probability}, 20\penalty0 (79):\penalty0
  1 -- 32, 2015.

\bibitem[Rao(2018)]{Rao2018AHI}
S.~Rao.
\newblock A {Hoeffding} inequality for {Markov} chains.
\newblock \emph{Electronic Communications in Probability}, 2018.

\bibitem[Read et~al.(2012)Read, Martino, and Luengo]{Read2012EfficientMC}
J.~Read, L.~Martino, and D.~Luengo.
\newblock Efficient {Monte Carlo} methods for multi-dimensional learning with
  classifier chains.
\newblock \emph{Pattern Recognit.}, 47:\penalty0 1535--1546, 2012.

\bibitem[Robert and Casella(2005)]{Robert2005MonteCS}
C.~P. Robert and G.~Casella.
\newblock Monte carlo statistical methods.
\newblock \emph{Technometrics}, 47:\penalty0 243 -- 243, 2005.

\bibitem[Truong(2020)]{Truong2020ReplicaAO}
L.~V. Truong.
\newblock Replica analysis of the linear model with {Markov} or hidden {Markov}
  signal priors.
\newblock \emph{ArXiv}, abs/2009.13370, 2020.

\bibitem[Truong(2021)]{Truong2021LinearMW}
L.~V. Truong.
\newblock Linear models with hidden {Markov} sources via replica method.
\newblock \emph{2021 IEEE International Symposium on Information Theory
  (ISIT)}, pages 396--401, 2021.

\bibitem[Truong(2022)]{Truong2022OnLM}
L.~V. Truong.
\newblock On linear model with markov signal priors.
\newblock In \emph{AISTATS}, 2022.

\bibitem[Tuominen and Tweedie(1979)]{TR1979}
P.~Tuominen and R.~L. Tweedie.
\newblock {Markov Chains with Continuous Components}.
\newblock \emph{Proceedings of the London Mathematical Society}, s3-38\penalty0
  (1):\penalty0 89--114, 01 1979.

\bibitem[Wolfer and Kontorovich(2019)]{WK19ALT}
G.~Wolfer and A.~Kontorovich.
\newblock Estimating the mixing time of ergodic {Markov} chains.
\newblock In \emph{32nd Annual Conference on Learning Theory}, 2019.

\bibitem[Zanella and Roberts(2019)]{Zanella2019ScalableIT}
G.~Zanella and G.~O. Roberts.
\newblock Scalable importance tempering and {Bayesian} variable selection.
\newblock \emph{Journal of the Royal Statistical Society: Series B (Statistical
  Methodology)}, 81, 2019.

\end{thebibliography}
